\documentclass[draftclsnofoot,onecolumn,peerreview]{IEEEtran}
\usepackage{color}
\usepackage{cite,graphicx,url,amssymb}
\usepackage[tight,footnotesize]{subfigure}
\usepackage[cmex10]{amsmath}
\newcommand{\note}[1]{}

\newtheorem{lemma}{Lemma}
\newtheorem{theorem}{Theorem}
\newtheorem{corollary}{Corollary}
\newtheorem{definition}{Definition}

\newtheorem{example}{Example}

\def \1{\mathbf 1}
\def \E{\mathbb E}
\def \P{\mathbb P}
\def \R{\mathbb R}
\def \H{{\cal H}}
\def \A{{\cal A}}

\def \X{{\cal X}}

\def \ch{c^*}
\def \mm{C}
\def \N{N}

\begin{document}

\title{The Geometry of Generalized Binary Search}%
\author{Robert~D.~Nowak
}

\maketitle

\begin{abstract}
  \sloppypar This paper investigates the problem of determining a binary-valued function through a sequence of strategically selected queries. The focus is an algorithm called Generalized Binary Search (GBS).  GBS is a well-known greedy algorithm for determining a binary-valued function through a sequence of strategically selected queries.  At each step, a query is selected that most evenly splits the hypotheses under consideration into two disjoint subsets, a natural generalization of the idea underlying classic binary search.  This paper develops novel incoherence and geometric  conditions under which GBS achieves the information-theoretically optimal query complexity; i.e., given a collection of $\N$ hypotheses, GBS terminates with the correct function after no more than a constant times $\log \N$ queries.  Furthermore, a noise-tolerant version of GBS is developed that also achieves the optimal query complexity.  These results are applied to learning halfspaces, a problem arising routinely in image processing and machine learning.  \end{abstract}


\section{Introduction}
\label{sec:intro}

This paper studies learning problems of the following form.  Consider a finite, but potentially very large, collection of binary-valued functions $\H$ defined on a domain $\X$.  In this paper, $\H$ will be called the {\em hypothesis space} and $\X$ will be called the {\em query space}.  Each $h \in \H$ is a mapping from $\X$ to $\{-1,1\}$.  Throughout the paper we will let $\N$ denote the cardinality of $\H$. Assume that the functions in $\H$ are unique and that one function, $h^* \in \H$, produces the correct binary labeling.  It is assumed that $h^*$ is fixed but unknown, and the goal is to determine $h^*$ through as few queries from $\X$ as possible.  For each query $x\in \X$, the value $h^*(x)$, possibly corrupted with independently distributed binary noise, is observed. The goal is to strategically select queries in a sequential fashion in order to identify $h^*$ as quickly as possible.  

\textcolor{black}{If the responses to queries are noiseless, then the problem is related to the construction of a binary decision tree. A sequence of queries defines a path from the root of the tree (corresponding to $\H$) to a leaf (corresponding to a single element of $\H$).  There are several ways in which one might define the notion of an optimal tree; e.g., the tree with the minimum average or worst case depth. 
In general the determination of the optimal tree (in either sense above) is a combinatorial problem and was shown by Hyafil and Rivest to be NP-complete \cite{tree4}.  Therefore, this paper investigates the performance of a greedy procedure called {\em generalized binary search} (GBS), depicted below in Fig.~\ref{fig:gbs}.  At each step GBS selects a query that results in the most even split of the hypotheses under consideration into two subsets responding $+1$ and $-1$, respectively, to the query. The correct response to the query eliminates one of these two subsets from further consideration.  We denote the number of hypotheses remaining at step $n$ by $|\H_n|$. The main results of the paper characterize the {\em worst-case} number of queries required by GBS in order to identify the correct hypothesis $h^*$. More formally, we define the notion of query complexity as follows.}

\textcolor{black}{\begin{definition} The minimum number of queries required by GBS (or another algorithm) to identify any hypothesis in $\H$ is called the {\em query complexity} of the algorithm.  The query complexity is said to be {\em near-optimal} if it is within a constant factor of $\log \N$, since at least $\log \N$ queries are required to specify one of $\N$ hypotheses. \\ 
\end{definition}}

\begin{figure}[t]
\fbox{\parbox[b]{3in}{{\underline{\bf Generalized Binary Search (GBS)}}  \\ \vspace{-.25in} \\
initialize: $n=0$, $\H_0 = \H$.\\
while $|\H_n|>1$  \\ \vspace{-.25in} \\
1) Select $x_n = \arg \min_{x \in \X} |\sum_{h \in \H_n} h(x)|$. \\
2) Query with $x_n$ to obtain response $y_n=h^*(x_n)$. \\
3) Set $\H_{n+1} = \{h\in \H_n: h(x_n)=y_n\}$, $n=n+1$.
}}
\caption{Generalized Binary Search, also known as the Splitting Algorithm.}
\label{fig:gbs}
\end{figure}

Conditions are established under which GBS (and a noise-tolerant variant) have a near-optimal query complexity. The main contributions of this paper are two-fold.  First, incoherence and geometric relations between the pair $(\X,\H)$ are studied to bound the number of queries required by GBS.  This leads to an easily verifiable sufficient condition that guarantees that GBS terminates with the correct hypothesis after no more than a constant times $\log\N$ queries.  Second, noise-tolerant versions of GBS are proposed. The following noise model is considered. The binary response $y\in\{-1,1\}$ to a query $x\in \X$ is an independent realization of the random variable $Y$ satisfying $\P(Y=h^*(x))>\P(Y=-h^*(x))$, where $\P$ denotes the underlying probability measure.  In other words, the response to $x$ is only probably correct.  If a query $x$ is repeated more than once, then each response is an independent realization of $Y$.  A new algorithm based on a weighted (soft-decision) GBS procedure is shown to confidently identify $h^*$ after a constant times $\log \N$ queries even in the presence of noise (under the sufficient condition mentioned above).  An agnostic algorithm that performs well even if $h^*$ is not in the hypothesis space $\H$ is also proposed.

\subsection{Notation}

The following notation will be used throughout the paper. The hypothesis space $\H$ is a finite collection of binary-valued functions defined on a domain $\X$, which is called the query space.  Each $h \in \H$ is a mapping from $\X$ to $\{-1,1\}$. For any subset $\H' \subset H$, $|\H'|$ denotes the number of hypotheses in $\H'$. The number of hypotheses in $\H$ is denoted by $N := |\H|$.  


\section{A Geometrical View of Generalized Binary Search}
\label{sec:gbs}
The efficiency of classic binary search is due to the fact at each step there exists a query that splits the pool of viable hypotheses in half.  
The existence of such queries is a result of the special ordered structure of the problem.  Because of ordering, optimal query locations are easily identified by bisection.  In the general setting in which the query and hypothesis space are arbitrary it is impossible to order the hypotheses in a similar fashion and ``bisecting'' queries may not exist.  For example, consider hypotheses associated with halfspaces of $\X=\R^d$.  Each hypothesis takes the value $+1$ on its halfspace and $-1$ on the complement. A bisecting query may not exist in this case.   To address such situations we next introduce a more general framework that does not require an ordered structure.



\subsection{Partition of $\X$}
While it may not be possible to naturally order the hypotheses within $\X$, there does exist a similar local geometry that can be exploited in the search process. Observe that the query space $\X$ can be partitioned into equivalence subsets such that every $h \in \H$ is constant for all queries in each such subset.  Let $\A(\X,\H)$ denote the smallest such partition\footnote{Each $h$ splits $\X$ into two disjoint sets. Let $C_h := \{x\in \X: h(x)=+1\}$ and let $\bar{C}_n$ denote its complement. $\A$ is the collection of all non-empty intersections of  the form $\bigcap_{h\in H} \widetilde{C}_h$, where $\widetilde{C}_h
  \in \{C_h,\bar C_h\}$, and it is the smallest partition that refines the sets $\{C_h\}_{h\in \H}$. $\A$ is known as the {\em join} of the sets  $\{C_h\}_{h\in \H}$.}. Note that $\X = \bigcup_{A \in \A} A$. For every $A \in \A$ and $h \in \H$, the value of $h(x)$ is constant (either $+1$ or $-1$) for all $x\in A$; denote this value by $h(A)$. Observe that the query selection step in GBS is equivalent to an optimization over the partition cells in $\A$.  That is, it suffices to select a partition cell for the query according to $A_n = \arg \min_{A \in \A} |\sum_{h \in \H_n} h(A)|$. 

The main results of this paper concern the query complexity of GBS, but before moving on let us comment on the computational complexity of the algorithm. The query selection step is the main computational burden in GBS.  \textcolor{black}{Constructing $\A$ may itself be computationally intensive.} However, given $\A$ the computational complexity of GBS is $\N \, |\A|$, up to a constant factor, where $|\A|$ denotes the number of partition cells
in $\A$.  The size and construction of $\A$ is manageable in many practical situations. For example, if $\X$ is finite, then $|\A| \leq |\X|$, where $|\X|$ is the cardinality of $\X$.  Later, in Section~\ref{linear}, we show that if $\H$ is defined by $\N$ halfspaces of $\X:=\R^d$, then $|\A|$ grows like  $\N^d$.  

\subsection{Distance in $\X$}
The partition $\A$ provides a geometrical link between $\X$ and $\H$. The hypotheses induce a distance function on $\A$, and hence $\X$.  For every pair $A,A'\in \A$ the Hamming distance between the response vectors $\{h_1(A),\dots,h_\N(A)\}$ and $\{h_1(A'),\dots,h_\N(A')\}$ provides a natural distance metric in $\X$.
\begin{definition} Two sets $A,A' \in \A$ are said to be {\em $k$-neighbors} if
  $k$ or fewer hypotheses (along with their complements, if they belong to $\H$) output different values on $A$ and $A'$.
\label{neighbor}
\end{definition}
For example, suppose that $\H$ is symmetric, so that $h \in \H$ implies $-h\in \H$.  Then two sets $A$ and $A'$ are $k$-neighbors if the Hamming distance between their respective response vectors is less than or equal to $2k$.  If $\H$ is non-symmetric ($h\in\H$ implies that $-h$ is not in $\H$), then $A$ and $A'$ are $k$-neighbors if the Hamming distance between their respective response vectors is less than or equal to $k$. 
\begin{definition}
  The pair $(\X,\H)$ is said to be {\em $k$-neighborly} if the
  $k$-neighborhood graph of $\A$ is connected (i.e., for every pair of
  sets in $\A$ there exists a sequence of $k$-neighbor sets that begins
  at one of the pair and ends with the other). \label{neighborly}
\end{definition}
\sloppypar If $(\X,\H)$ is $k$-neighborly, then the distance between $A$ and $A'$ is bounded by $k$ times the minimum path length between $A$ and $A'$.  Moreover, the neighborly condition implies that there is an incremental way to move from one query to the another, moving a distance of at most $k$ at each step.  This local geometry guarantees that near-bisecting queries almost always exist, as shown in the following lemma.


\sloppypar
\begin{lemma}{\em
Assume that $(\X,\H)$ is $k$-neighborly and define the coherence parameter
\begin{eqnarray}
\ch(\X,\H) & := & \min_P \max_{h\in \H} \left|\sum_{A\in \A} h(A) \, P(A)\right| \ , 
\label{cstar}
\end{eqnarray}
where the minimization is over all probability mass functions on $\A$.  For every $\H' \subset \H$ and any constant $c$ satisfying $c^* \leq  c < 1$ there exists an $A \in \A$ that approximately bisects $\H'$
$$\left|\sum_{h \in \H'} h(A)\right| \ \leq \ c \, |\H'| \ , $$
or the set $\H'$ is a small
$$|\H'| < \frac{k}{c} \ ,$$
where $|\H'|$ denotes the cardinality of $\H'$.}
\label{lemma1}
\end{lemma}

{\em Proof:} According to
the definition of $\ch$ it follows that there exists a probability distribution $P$ such that
$$\left|\sum_{h \in \H'} \sum_{A\in \A} h(A)P(A)\right| \ \leq \ c\, |\H'| \ . $$
This implies that there exists an $A \in \A$ such that
$$\left|\sum_{h \in \H'} h(A)\right| \ \leq \ c\, |\H'| \ , $$
or there exists a pair $A$ and $A'$ such that
$$\sum_{h \in \H'}  h(A) \ > \ c\, |\H'|  \mbox{ \ \ and }  \sum_{h \in \H'}  h(A') \ < \ -c\, |\H'| \ . $$
In the former case, it follows that a query from $A$ will reduce the size of $\H'$ by a factor of at least $(1+c)/2$ (i.e., every query $x \in \A$ approximately bisects the subset $\H'$). In latter case, an approximately bisecting query does not exist, but the $k$-neighborly condition implies that $|\H'|$ must be small.  To see this note that the $k$-neighborly condition guarantees that there exists a sequence of $k$-neighbor sets beginning at $A$ and ending at $A'$.  By assumption in this case, $\left| \sum_{h \in
    \H'}h(\cdot)\right| >c|\H'|$ on every set and the sign of $\sum_{h \in \H'}h(\cdot)$ must change at some point in the sequence.  It follows that there exist $k$-neighbor sets $A$ and $A'$ such that $\sum_{h \in \H'} h(A) \ > \ c|\H'|$ and $\sum_{h \in \H'} h(A') \ < \ -c|\H'|$.  Two inequalities follow from this observation.  First, $\sum_{h \in
    \H'}h(A)-\sum_{h\in \H'}h(A')> 2c|\H'|$.  Second, $|\sum_{h \in
    \H'}h(A)-\sum_{h\in \H'}h(A')| \leq 2k$.  Note that if $h$ and its complement $h'$ belong to $\H'$, then their contributions to the quantity $|\sum_{h \in
    \H'}h(A)-\sum_{h\in \H'}h(A')|$ cancel each other. Combining these
  inequalities yields $|\H'|<k/c$. \hfill $\blacksquare$

\textcolor{black}{\begin{example}
{\em To illustrate Lemma~\ref{lemma1}, consider the special situation in which we are given two points $x_1,x_2 \in \X$ known to satisfy $h^*(x_1)=+1$ and $h^*(x_2)=-1$.  This allows us to restrict our attention to only those hypotheses that agree with $h^*$ at these points.  Let $\H$ denote this collection of hypotheses. A depiction of this situation is shown in Fig.~\ref{fig1}, where the solid curves represent the classification boundaries of the hypotheses, and each cell in the partition shown corresponds to a subset of $\X$ (i.e., an element of $\A$).  As long as each subset is non-empty, then the $1$-neighborhood graph is connected in this example. The minimization in (\ref{cstar}) is achieved by the distribution $P = \frac{1}{2} \delta_{x_1} + \frac{1}{2} \delta_{x_2}$ (equal point-masses on $x_1$ and $x_2$) and  $\ch(\X,\H) = 0$.  Lemma~\ref{lemma1} implies that there exists a query (equivalently a partition cell $A$) where half of the hypotheses take the value $+1$ and the other half $-1$.  The shaded cell in Fig.~\ref{fig1} has this {\em bisection} property.  The figure also shows a dashed path between $x_1$ and $x_2$ that passes through the bisecting cell.} \vspace{-.24in}
\end{example}}

\begin{figure}[h]
\centering
 \includegraphics[width=10cm]{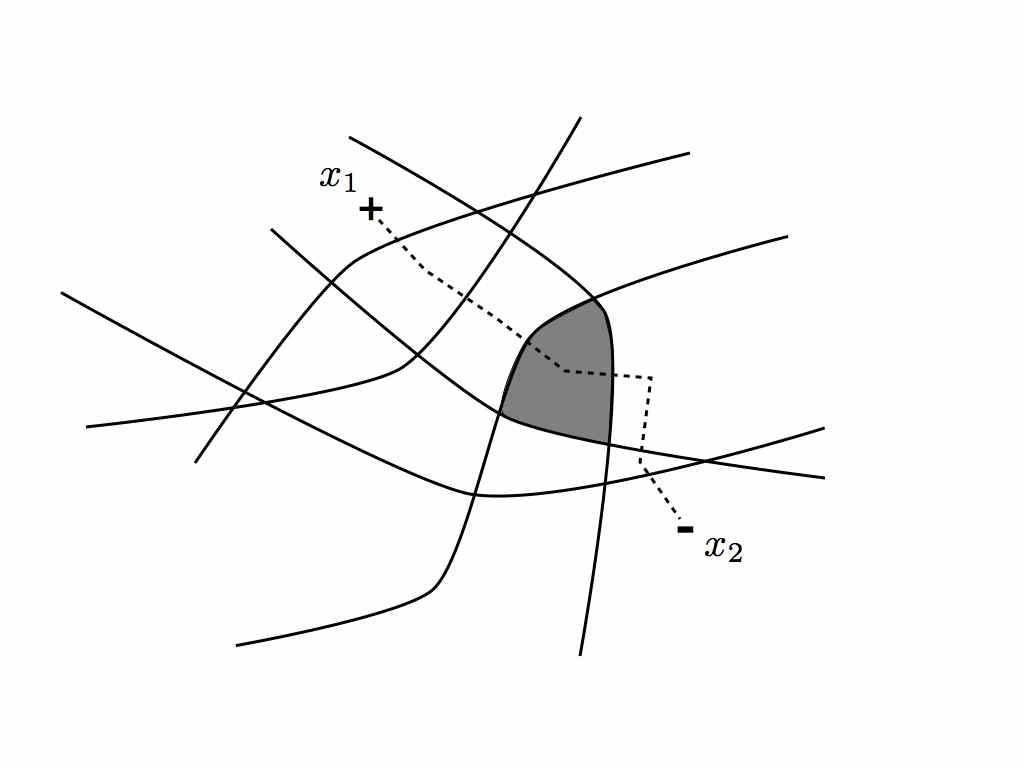}  \ \vspace{-.5in}
    \caption{An illustration of the idea of GBS.  Each solid curve denotes the decision boundary of a hypothesis.  There are six boundaries/hypotheses in this example. The correct hypothesis in this case is known to satisfy $h^*(x_1)=+1$ and $h^*(x_2)=-1$. Without loss of generality we may assume that all hypotheses agree with $h^*$ at these two points. The dashed path between the points $x_1$ and $x_2$ reveals a bisecting query location.  As the path crosses a decision boundary the corresponding hypothesis changes its output from $+1$ to $-1$ (or vice-versa, depending on the direction followed).  At a certain point, indicated by the shaded cell, half of the hypotheses output $+1$ and half output $-1$.  Selecting a query from this cell will {\em bisect} the collection of hypotheses. }
    \label{fig1}
\end{figure}

\subsection{Coherence and Query Complexity}
The coherence parameter $\ch$ quantifies the informativeness of queries.  The coherence parameter is optimized over the choice of $P$, rather than sampled at random according to a specific distribution on $\X$, because the queries may be selected as needed from $\X$.  The minimizer in (\ref{cstar}) exists because the minimization can be computed over the space of finite-dimensional probability mass functions over the elements of $\A$.  For $\ch$ to be close to $0$, there must exist a distribution $P$ on $\A$ so that the moment of every $h\in \H$ is close to zero (i.e., for each $h \in \H$ the probabilities of the responses $+1$ and $-1$ are both close to $1/2$).  This implies that there is a way to randomly sample queries so that the expected response of every hypothesis is close to zero.  In this sense, the queries are incoherent with the hypotheses. In Lemma~\ref{lemma1}, $\ch$ bounds the proportion of the split of {\em any} subset $\H'$ generated by the best query  (i.e., the degree to which the best query bisects any subset $\H'$).  The coherence parameter $c^*$ leads to a bound on the number of queries required by GBS.


\sloppypar
\begin{theorem}{\em
  If $(\X,\H)$ is $k$-neighborly, then GBS terminates
    with the correct hypothesis after at most $\lceil \log
    \N/\log(\lambda^{-1})\rceil$ queries, where $\lambda =
    \max\{\frac{1+\ch}{2},\frac{k+1}{k+2}\}$.}
\label{thm1}
\end{theorem}

{\em Proof:} Consider the $n$th step of the GBS algorithm.  Lemma~\ref{lemma1} shows that for any $c \in [\ch,1)$ either there exists an approximately bisecting query and $|\H_n|\leq \frac{1+c}{2} |\H_{n-1}|$ or $|\H_{n-1}|<k/c$.  The uniqueness of the hypotheses with respect to $\X$ implies that there exists a query that eliminates at least one hypothesis. Therefore, $|H_n|\leq |\H_{n-1}|-1 = |\H_{n-1}|(1-|\H_{n-1}|^{-1}) < |\H_{n-1}|(1-c/k)$.  It follows that each GBS query reduces the number of viable hypotheses by a factor of at least
  $$\lambda \ := \ \min_{c\geq \ch} \max \left\{\frac{1+c}{2},1-c/k\right\} =
  \max\left\{\frac{1+\ch}{2},\frac{k+1}{k+2}\right\} \ .$$ 
Therefore, $|\H_n| \leq \N \lambda^n$ and GBS is guaranteed to terminate when $n$
satisfies $\N \lambda^n \leq 1$.  Taking the logarithm of this inequality produces the query complexity bound.
\hfill $\blacksquare$

  Theorem~\ref{thm1} demonstrates that if $(\X,\H)$ is neighborly, then the query complexity of GBS is near-optimal; i.e., within a constant factor of $\log_2\N$. The constant depends on coherence parameter $\ch$ and $k$, and clearly it is desirable that both are as small as possible. Note that GBS does not require knowledge of $\ch$ or $k$.  We also remark that the constant in the bound is not necessarily the best that one can obtain.  The proof involves selecting  $c$ to balance splitting factor $\frac{1+c}{2}$ and the ``tail'' behavior $1-c/k$, and this may not give the best bound.  The coherence parameter $\ch$ can be computed or bounded for many pairs $(\X,\H)$ that are commonly encountered in applications, as covered later in Section~\ref{coherence}.



\section{Noisy Generalized Binary Search}
\label{sec:ngbs}

In noisy problems, the search must cope with erroneous responses.  Specifically, assume that for any query $x\in \X$ the binary response $y\in\{-1,1\}$ is an independent realization of the random variable $Y$ satisfying $\P(Y=h^*(x))>\P(Y=-h^*(x))$ (i.e., the response is only probably correct).  If a query $x$ is repeated more than once, then each response is an independent realization of $Y$. Define the {\em noise-level} for the query $x$ as $\alpha_x := \P(Y=-h^*(x))$.  Throughout the paper we will let $\alpha := \sup_{x\in \X} \alpha_x$ and assume that $\alpha < 1/2$.  Before presenting the main approach to noisy GBS, we first consider a simple strategy based on repetitive querying that will serve as a benchmark for comparison.

\subsection{Repetitive Querying}
We begin by describing a simple noise-tolerant version of GBS. The noise-tolerant algorithm is based on the simple idea of repeating each query of the GBS several times, in order to overcome the uncertainty introduced by the noise.  Similar approaches are proposed in the work K{\" a}{\" a}ri{\" a}inen \cite{matti}. Karp and Kleinberg \cite{kk:07} analyze of this strategy for noise-tolerant classic binary search. This is essentially like using a simple repetition code to communicate over a noisy channel. This procedure is termed noise-tolerant GBS (NGBS) and is summarized in Fig.~\ref{fig:ngbs}.

\begin{figure}[h]
\fbox{\parbox[b]{3in}{{\underline{\bf Noise-Tolerant GBS (NGBS)}}  \\ \vspace{-.25in} \\
input: $\H$, repetition rate $R\geq 1$ (integer). \\
initialize: $n=0$, $\H_0 = \H.$ \\
while $| \H_n|>1$  \\ \vspace{-.25in} \\
1) Select $A_n = \arg \min_{A \in \A} |\sum_{h \in \H_n} h(A)|$. \\
2) Query $R$ times from $A_n$ to obtain $R$ noisy versions of $y_n=h^*(A_n)$. Let $\widehat{y}_n$ denote the majority vote of the noisy responses. \\
3) Set $\H_{n+1} = \{h\in \H_n: h(A_n)=\widehat y_n\}$, $n=n+1$.
}}
\caption{Noise-tolerant GBS based on repeated queries.}
\label{fig:ngbs}
\end{figure}
\begin{theorem}\textcolor{black}{\em
Let $n_0$ denote the number of queries made by GBS to determine $h^*$ in the noiseless setting. Then in the noisy setting, with probability at least $\max\{0 \, , \, 1-n_0 \, e^{-R|\frac12-\alpha|^2}\} $  the noise-tolerant GBS algorithm in Fig.~\ref{fig:ngbs} terminates in exactly $R \, n_0$ queries and outputs $h^*$.}
\label{thm:ngbs}
\end{theorem}
\begin{proof} Consider a specific query $x\in\X$ repeated $R$ times, let $\widehat{p}$ denote the frequency of $+1$ in the $R$ trials, and let $p = \E[\widehat{p}]$.  The majority vote decision is correct if $|\widehat p-p|\leq \frac12-\alpha$.  By Chernoff's bound we have $\P(|\widehat p - p|\geq \frac12-\alpha) \leq 2 e^{-2R|\frac12-\alpha|^2}$.  The results follows by the union bound. \end{proof}

Based on the bound above, $R$ must satisfy $R\geq \frac{\log(n_o/\delta)}{|1/2-\alpha|^2}$ to guarantee that the labels determined for all $n_0$ queries are correct with probability $1-\delta$.  The query complexity of NGBS can thus be bounded by $\frac{n_0\log(n_0/\delta)}{|1/2-\alpha|^2}.$ Recall that $\N = |\H|$, the cardinality of $\H$.  If $n_0 = \log \N$, then bound on the query complexity of NGBS is proportional to $\log \N \, \log \frac{\log \N}{\delta}$, a logarithmic factor worse than the query complexity in the noiseless setting. Moreover, if an upper bound on $n_0$ is not known in advance, then one must assume the worst-case value, $n_0 = \N$, in order to set $R$.  That is, in order to guarantee that the correct hypothesis is determined with probability at least $1-\delta$, the required number of repetitions of each query is $R = \lceil\frac{\log(\N/\delta)}{|1/2-\alpha|^2}\rceil.$ In this situation, the bound on the query complexity of NGBS is proportional to $\log \N \, \log\frac{N}{\delta}$, compared to $\log N$ in the noiseless setting.  It is conjectured that the extra logarithmic factor cannot be removed from the query complexity (i.e., it is unavoidable using repetitive queries). As we show next, these problems can be eliminated by a more sophisticated approach to noisy GBS.

\subsection{Soft-Decision Procedure}
A more effective approach to noisy GBS is based on the following soft-decision procedure. A similar procedure has been shown to be near-optimal for the noisy (classic) binary search problem by Burnashev and Zigangirov \cite{bz} and later independently by Karp and Kleinberg \cite{kk:07}.  The crucial distinction here is that GBS calls for a 
more general approach to query selection and a fundamentally different convergence analysis.
Let $p_0$ be a known probability measure over $\H$.  That is, $p_0:\H
\rightarrow [0,1]$ and $\sum_{h\in\H} p_0(h) = 1$.  The measure $p_0$
can be viewed as an initial weighting over the hypothesis class. For example, taking $p_0$ to be the uniform distribution over $\H$
expresses the fact that all hypothesis are equally reasonable prior
to making queries.  We will assume that $p_0$ is uniform for the remainder of
the paper, but the extension to other initial distributions is trivial. Note, however, that we still assume that $h^* \in \H$ is fixed but unknown. After each query and response $(x_n,y_n),$
$n=0,1,\dots,$ the distribution is updated according to
\begin{eqnarray}
p_{n+1}(h) & \propto & p_n(h) \, \beta^{(1-z_n(h))/2}
(1-\beta)^{(1+z_n(h))/2},
\label{update1}
\end{eqnarray}
where $z_n(h) = h(x_n)y_n, \ h\in\H$, $\beta$ is any constant satisfying $0 < \beta < 1/2$, and $p_{n+1}(h)$ is normalized to satisfy $\sum_{h\in\H} p_{n+1}(h) = 1$ . The update can be viewed as an application of Bayes rule and its effect is simple; the probability masses of hypotheses that agree with the label $y_n$ are boosted relative to those that disagree. The parameter $\beta$ controls the size of the boost.  The hypothesis with the largest weight is selected at each step: $$\widehat{h}_n := \arg \max_{h\in\H} p_n(h) . $$ If the maximizer is not unique, one of the maximizers is selected at random. Note that, unlike the hard-decisions made by the GBS algorithm in Fig.~\ref{fig:gbs}, this procedure does not eliminate hypotheses that disagree with the observed labels, rather the weight assigned to each hypothesis is an indication of how successful its predictions have been.  Thus, the procedure is termed Soft-Decision GBS (SGBS) and is summarized in Fig.~\ref{fig:sgbs}.

The goal of SGBS is to drive the error $\P(\widehat{h}_n \neq h^*)$ to zero as quickly as possible by strategically selecting the queries.  The query selection at each step of SGBS must be informative with respect to the distribution $p_{n}$. In particular, if the {\em weighted prediction} $\sum_{h\in H} p_n(h) h(x)$ is close to zero for a certain $x$ (or $A$), then a label at that point is informative due to the large disagreement among the hypotheses.  If multiple $A \in \A$ minimize $|\sum_{h \in \H}  p_n(h) h(A)|$, then
one of the minimizers is selected uniformly at random.

\begin{figure}[h]
\fbox{\parbox[b]{3in}{{\underline{\bf Soft-Decision Generalized Binary Search (SGBS)}}  \\ \vspace{-.25in} \\
initialize: $p_0$ uniform over $\H$.\\
for $n=0,1,2,\dots$ \\
1) $A_n = \arg \min_{A \in \A} |\sum_{h \in \H}  p_n(h) h(A)|$. \\
2) Obtain noisy response $y_n$. \\
3) Bayes update $p_n \rightarrow p_{n+1}$; Eqn.~(\ref{update1}). \\
hypothesis selected at each step: \\
$\widehat{h}_n := \arg\max_{h\in H} p_n(h)$ \\
}}
\caption{Soft-Decision algorithm for noise-tolerant GBS.}
\label{fig:sgbs}
\end{figure}

To analyze SGBS, define $\mm_n := (1-p_n(h^*))/p_n(h^*)$, $n\geq 0$. The variable $\mm_n$ was also used by Burnashev and Zigangirov \cite{bz} to analyze classic binary search. It reflects the amount of mass that $p_n$ places on incorrect hypotheses.  Let $\P$ denotes the underlying probability measure governing noises and possible randomization in query selection, and let $\E$ denote expectation with respect to $\P$. Note that by Markov's inequality \begin{eqnarray} \P(\widehat h_n \neq h^*) & \leq & \P(p_n(h^*) < 1/2) \ = \ \P(\mm_n > 1) \ \leq \ \E[\mm_n] \ .  \label{markov} \end{eqnarray}
At this point, the method of analyzing SGBS departs from that of Burnashev and Zigangirov \cite{bz} which focused only on the classic binary search problem.  The lack of an ordered structure calls for a different attack on the problem, which is summarized in the following results and detailed in the Appendix.

\begin{lemma} {\em Consider any sequence of queries $\{x_n\}_{n\geq 0}$ and the corresponding responses $\{y_n\}_{n\geq 0}$. If $\beta \geq \alpha$, then $\{\mm_n\}_{n\geq 0}$ is a nonnegative supermartingale with respect to $\{p_n\}_{n\geq 0}$; i.e., $\E[C_{n+1}|p_n] \leq C_n$ for all $n\geq 0$.} \label{martingale} \end{lemma}


The lemma is proved in the Appendix.
The condition $\beta\geq\alpha$ ensures that the update (\ref{update1}) is not overly aggressive.  It follows that $\E[C_n] \leq C_0$ and by the Martingale Convergence Theorem we have that $\lim_{n\rightarrow \infty} C_n$ exists and is finite (for more information on Martingale theory one can refer to the textbook by Br\'{e}maud \cite{bremaud}).   Furthermore, we have the following theorem.
\textcolor{black}{\begin{theorem}{\em
\textcolor{black}{Consider any sequence of queries $\{x_n\}_{n\geq 0}$ and the corresponding responses $\{y_n\}_{n\geq 0}$. If $\beta > \alpha$, then $\lim_{n\rightarrow \infty} \P(\widehat h_n \neq h^*) \leq C_0$.}}
\label{thm2}
\end{theorem}
\begin{proof}
First observe that for every positive integer $n$
\begin{eqnarray*}
  \E[\mm_n] & = & \E[(\mm_n/\mm_{n-1}) \, \mm_{n-1}] \ = \ \E\left[\E[(\mm_n/\mm_{n-1}) \, \mm_{n-1}|p_{n-1}]\right] \\
 & = & \E\left[\mm_{n-1} \,  \E[(\mm_n/\mm_{n-1})|p_{n-1}]\right] \ \leq \ \E[\mm_{n-1}] \, \max_{p_{n-1}}\E[(\mm_n/\mm_{n-1})|p_{n-1}] \\
& \leq & \mm_0 \left(\max_{i = 0,\dots,n-1} \max_{p_i}\, \E[(\mm_{i+1}/\mm_{i})|p_i]\right)^n \ .
\end{eqnarray*}
In the proof of Lemma~\ref{martingale}, it is shown that if $\beta > \alpha$, then $\E[(\mm_{i+1}/\mm_{i})|p_i]\leq1$ for every $p_i$ and therefore $\max_{p_i}\E[(\mm_{i+1}/\mm_{i})|p_i]~\leq~1$.
It follows that the sequence $a_n := \left(\max_{i = 0,\dots,n-1} \max_{p_i}\, \E[(\mm_{i+1}/\mm_{i})|p_i]\right)^n$ is monotonically decreasing.  The
result follows from (\ref{markov}).
\end{proof}}

Note that if we can determine an upper bound for the sequence $\max_{p_i}\E[(\mm_{i+1}/\mm_{i})|p_i]\leq 1-\lambda < 1$, $i=0,\dots,n-1$, then it follows that that $\P(\widehat h_n \neq h^*) \leq \N (1-\lambda)^n \leq \N e^{-\lambda n}$. 
\textcolor{black}{Unfortunately, the SGBS algorithm in Fig.~\ref{fig:sgbs} does not readily admit such a bound. To obtain a bound, the query selection criterion is randomized.  A similar randomization technique has been successfully used by Burnashev and Zigangirov \cite{bz} and Karp and Kleinberg \cite{kk:07} to analyze the noisy (classic) binary search problem, but again the lack of an ordering in the general setting requires a different analysis of GBS.  }

The modified SGBS algorithm is outlined in Fig.~\ref{fig:msgbs}. It is easily verified
that Lemma~\ref{martingale} and Theorem~\ref{thm2} also hold for the modified SGBS algorithm. This follows since the modified query selection step is identical to that of the original SGBS algorithm, unless there exist two neighboring sets with strongly bipolar weighted responses.  In the latter case, a query is randomly selected from one of these two sets with equal probability.

\begin{figure}[h]
\fbox{\parbox[b]{6in}{{\underline{\bf Modified SGBS}}  \\ \vspace{-.25in} \\
      initialize: $p_0$ uniform over $\H$.\\
      for $n=0,1,2,\dots$ \\ \vspace{-.1in} \\
      1) \ \mbox{\parbox[t]{5.5in}{Let $b = \min_{A\in \A}
          |\sum_{h \in \H} p_n(h) h(A)|$.  If there exist $1$-neighbor
          sets $A$ and $A'$ with $\sum_{h \in \H} p_n(h) h(A) > b$ and
          $\sum_{h \in \H} p_n(h) h(A') < -b$ , then select $x_n$ from
          $A$ or $A'$ with probability $1/2$ each. Otherwise select
          $x_n$ from the set $A_{\mbox{\tiny min}} = \arg \min_{A \in
            {\cal A}} |\sum_{h \in \H} p_n(h) h(A)|$. In the case
          that the sets above are non-unique, choose at random any
          one satisfying the requirements.}}\\ \vspace{-.0in} \\
      2) Obtain noisy response $y_n$. \\ 
       3) Bayes update $p_n \rightarrow p_{n+1}$; Eqn.~(\ref{update1}). \\ 
      hypothesis selected at each step: \\
      $\widehat{h}_n := \arg\max_{h\in H} p_n(h)$ }}
\caption{Modified SGBS Algorithm.}
\label{fig:msgbs}
\end{figure}

For every $A \in \A$ and any probability measure $p$ on $\H$ the {\em
  weighted prediction} on $A$ is defined to be $W(p,A) := \sum_{h\in
  H} p(h) h(A)$, where $h(A)$ is the constant value of $h$ for every
$x\in A$.  The following lemma, which is the soft-decision analog of Lemma~\ref{lemma1}, plays a crucial role in the analysis of
the modified SGBS algorithm.

\begin{lemma}{\em
  If $(\X,\H)$ is $k$-neighborly, then for every probability measure $p$
  on $\H$ there either exists a set $A \in \A$ such that $|W(p,A)|
  \leq \ch$ or a pair of $k$-neighbor sets $A,A' \in \A$ such that
  $W(p,A)>\ch$ and $W(p,A')<-\ch$.}
\label{lemma2}
\end{lemma} 
\begin{proof} Suppose that $\min_{A \in \A} |W(p,A)| > \ch$.  Then there must
  exist $A,A' \in \A$ such that $W(p,A) > \ch$ and $W(p,A') < -\ch$,
  otherwise $\ch$ cannot be the incoherence parameter of $\H$, defined in (\ref{cstar}).  To see this
  suppose, for instance, that $W(p,A) > \ch$ for all $A\in \A$. Then
  for every distribution $P$ on $\X$ we have $\sum_{A\in \A} \sum_{h\in \H}
  p(h)h(A) P(A)>\ch$.  This contradicts the definition of $\ch$ since
  $\sum_{A\in \A} \sum_{h\in \H} p(h)h(A) P(A) \leq \sum_{h \in \H} p(h)
  |\sum_{A\in \A} h(A) \, P(A)| \leq \max_{h\in \H} |\sum_{A \in \A} h(A) \,
  P(A)|$.    The neighborly condition guarantees that
  there exists a sequence of $k$-neighbor sets beginning at $A$ and
  ending at $A'$.  Since $|W(p,A)| >
  \ch$ on every set and the sign of $W(p,\cdot)$ must change at some
  point in the sequence, it follows that there exist $k$-neighbor
  sets satisfying the claim. 
\end{proof}

The lemma guarantees that there exists either a set in $\A$ on which the weighted hypotheses significantly disagree (provided $\ch$ is significantly below $1$) or two neighboring sets in $\A$ on which the weighted predictions are strongly bipolar. In either case, if a query is drawn randomly from these sets, then the weighted predictions are highly variable or uncertain, with respect to $p$. This makes the resulting label informative in either case.  If $(\X,\H)$ is $1$-neighborly, then the modified SGBS algorithm guarantees that $\P(\widehat h_n \neq h^*) \rightarrow 0$ exponentially fast. The $1$-neighborly condition is required so that the expected boost to $p_n(h^*)$ is significant at each step.  If this condition does not hold, then the boost could be arbitrarily small due to the effects of other hypotheses.  Fortunately, as shown in Section~\ref{coherence}, the $1$-neighborly condition holds in a wide range of common situations.

\begin{theorem} {\em Let $\P$ denote the underlying probability measure (governing
  noises and algorithm randomization). 
  If $\beta > \alpha$ and $(\X,\H)$ is $1$-neighborly, then the modified
  SGBS algorithm in Fig.~\ref{fig:msgbs} generates a sequence of hypotheses satisfying
$$\P(\widehat h_n \neq h^*) \ \leq \ \N \, (1-\lambda)^n \ \leq
\ \N \, e^{-\lambda n} \ \ , \ n=0,1,\dots$$
with exponential constant  $\lambda = \min\left\{\frac{1-\ch}{2},\frac{1}{4}\right\}\left(1-\frac{\beta(1-\alpha)}{1-\beta}-\frac{\alpha(1-\beta)}{\beta}\right)$, where $\ch$ is defined in (\ref{cstar}).}
\label{thm3}
\end{theorem}

The theorem is proved in the Appendix.
The exponential convergence rate\footnote{Note that the factor
  $\left(1-\frac{\beta(1-\alpha)}{1-\beta}-\frac{\alpha(1-\beta)}{\beta}\right)$
  in the exponential rate parameter $\lambda$ is a positive
  constant strictly less than $1$. For a noise level $\alpha$ this
  factor is maximized by a value $\beta \in (\alpha,1/2)$ which tends
  to $(1/2+\alpha)/2$ as $\alpha$ tends to $1/2$.} is governed by the coherence parameter $0\leq \ch < 1$.  As shown in Section~\ref{coherence} , the value of $\ch$ is typically a small constant much less than $1$ that is independent of the size of $\H$. In such situations, the query complexity of modified SGBS is near-optimal.  The query complexity of the modified SGBS algorithm can be derived as follows.  Let $\delta>0$ be a pre-specified confidence parameter.  The number of queries required to ensure that $\P(\widehat h_n \neq h^*) \leq \delta$ is $n \geq \lambda^{-1} \log \frac{\N}{\delta}$, which is near-optimal.  Intuitively, about $\log \N$ bits are required to encode each hypothesis. More formally, the noisy classic binary search problem satisfies the assumptions of Theorem~\ref{thm3} (as shown in Section~\ref{cbs}), and hence it is a special case of the general problem. Using information-theoretic methods, it has been shown by Burnashev and Zigangirov \cite{bz} (also see the work of Karp and Kleinberg \cite{kk:07}) that the query complexity for noisy classic binary search is also within a constant factor of $\log\frac{\N}{\delta}$.  In contrast, the query complexity bound for  NGBS, based on repeating queries, is at least logarithmic factor worse.  We conclude this section with an example applying Theorem~\ref{thm3} to the halfspace learning problem.

\begin{example}{\em 
Consider learning multidimensional halfspaces.  Let $\X = \R^d$ and consider hypotheses of the form
\begin{eqnarray}
h_i(x) := \mbox{sign}(\langle a_i, x\rangle+b_i) \ .
\label{halfspace}
\end{eqnarray}
where  $a_i \in \R^d$ and $b_i \in \R$ parameterize the hypothesis $h_i$ and $\langle a_i, x\rangle$ is the inner product in $\R^d$.  The following corollary characterizes the query complexity for this problem.
\begin{corollary}{
  Let $\H$ be a finite collection of hypotheses of form
  (\ref{halfspace}) and assume that the responses to each query are noisy, with noise bound $\alpha < 1/2$.  Then the hypotheses
  selected by modified SGBS with $\beta>\alpha$ satisfy
  $$\P(\widehat h_n \neq h^*) \ \leq \ \N \, e^{-\lambda n} \ ,$$
  with $\lambda =
  \frac{1}{4}\left(1-\frac{\beta(1-\alpha)}{1-\beta}-\frac{\alpha(1-\beta)}{\beta}\right)$.
  Moreover, $\widehat h_n$ can be computed in time polynomial in $|H|$. }
\label{thmS}
\end{corollary}

The error bound follows immediately from Theorem~\ref{thm3} since $\ch=0$ and $(\R^d,\H)$ is $1$-neighborly, as shown in Section~\ref{linear}.  The polynomial-time computational complexity follows from the work of Buck \cite{buck:43}, as discussed in Section~\ref{linear}. Suppose that $\H$ is an $\epsilon$-dense set with respect to a uniform probability measure on a ball in $\R^d$ (i.e., for {\em any} hyperplane of the form (\ref{halfspace}) $\H$ contains a hypothesis whose probability of error is within $\epsilon$ of it). The size of such an $\H$ satisfies $\log\N \leq C\, d \log \epsilon^{-1}$, for a constant $C>0$, which is the proportional to the minimum query complexity possible in this setting, as shown by Balcan et al \cite{nina:07}.  Those authors also present an algorithm with roughly the same query complexity for this problem. However, their algorithm is specifically designed for the linear threshold problem. Remarkably, near-optimal query complexity is achieved in polynomial-time by the general-purpose modified SGBS algorithm.}
\end{example}

\section{Agnostic GBS}

So far we have assumed that the correct hypothesis $h^*$ is in $\H$.  In this section we drop this assumption and consider {\em agnostic} algorithms guaranteed to find the best hypothesis in $\H$ even if the correct hypothesis $h^*$ is not in $\H$ and/or the assumptions of Theorem~\ref{thm1} or \ref{thm3} do not hold. The best hypothesis in $\H$ can be defined as the one that minimizes the error with respect to a probability measure on $\X$, denoted by $P_\X$, which can be arbitrary. This notion of ``best'' commonly arises in machine problems where it is customary to measure the error or {\em risk} with respect to a distribution on $\X$.  A common approach to hypothesis selection is {\em empirical risk minimization} (ERM), which uses queries randomly drawn according to $P_\X$ and then selects the hypothesis in $\H$ that minimizes the number of errors made on these queries. Given a budget of $n$ queries, consider the following agnostic procedure. Divide the query budget into three equal portions.  Use GBS (or NGBS or modified SGBS) with one portion, ERM (queries randomly distributed according $P_\X$) with another, and then allocate the third portion to queries from the subset of $\X$ where the hypothesis selected by GBS (or NGBS or modified SGBS) and the hypothesis selected by ERM disagree, with these queries randomly distributed according to the restriction of $P_\X$ to this subset.  Finally, select the hypothesis that makes the fewest mistakes on the third portion as the final choice.  The sample complexity of this agnostic procedure is within a constant factor of that of the better of the two competing algorithms.  For example, if the conditions of Theorems~\ref{thm1}~or~\ref{thm3} hold, then the sample complexity of the agnostic algorithm is proportional to $\log N$.  In general, the sample complexity of the agnostic procedure is within a constant factor of that of ERM alone.  We formalize this as follows.

\begin{lemma}
\label{runoff}{\em
Let $P_\X$ denote a probability measure on $\X$ and for every $h\in \H$ let $R(h)$ denote its probability of error with respect to $P_\X$.  Consider two hypotheses $h_1,h_2 \in \H$ and let $\Delta \subset \X$ denote the subset of queries for which $h_1$ and $h_2$ disagree; i.e., $h_1(x)\neq h_2(x)$ for all $x\in \Delta$.  Suppose that $m$ queries are drawn independently from $P_{\Delta}$, the restriction of $P_\X$ to the set $\Delta$, let $\widehat{R}_\Delta(h_1)$ and $\widehat{R}_\Delta(h_2)$ denote average number of errors made by $h_1$ and $h_2$ on these queries, let $R_\Delta(h_1) = \E[\widehat{R}_\Delta(h_1)]$ and $R_\Delta(h_2) = \E[\widehat{R}_\Delta(h_2)]$, and select $\widehat{h} = \arg \min\{\widehat{R}_\Delta(h_1),\widehat{R}_\Delta(h_2)\}$.  Then $R(\widehat{h}) >  \min\{R(h_1),R(h_2)\}$ with probability less than $e^{-m|R_\Delta(h_1)-R_\Delta(h_2)|^2/2}$.}
\end{lemma}
\begin{proof}
Define $\delta := R_\Delta(h_1)-R_\Delta(h_2)$ and let $\widehat{\delta} = \widehat R_\Delta(h_1)-\widehat R_\Delta(h_2)$.  By Hoeffding's inequality we have $\widehat{\delta} \in [\delta-\epsilon,\delta+\epsilon]$ with probability at least $1-2e^{-m|\epsilon|^2/2}$. 
It follows that $\P(\mbox{sign}(\widehat \delta) \neq \mbox{sign}(\delta)) \leq 2e^{-m|\delta|^2/2}$. For example, if $\delta >0$ then since
$\P(\widehat \delta < \delta-\epsilon) \leq 2e^{-m\epsilon^2/2}$ we may take $\delta = \epsilon$ to obtain $\P(\widehat \delta <0) \leq 2e^{-m\delta^2/2}$.  The result follows
since $h_1$ and $h_2$ agree on the complement of $\Delta$.
\end{proof}
Note that there is a distinct advantage to drawing queries from $P_\Delta$ rather
than $P_\X$, since the error exponent is proportional to $|R_\Delta(h_1)-R_\Delta(h_2)|^2$ which is greater than or equal to $|R(h_1)-R(h_2)|^2$.
Now to illustrate the idea, consider an agnostic procedure based on modified SGBS and
ERM.  The following theorem is proved in the Appendix.
\begin{theorem}
\label{thm5}
{\em
  Let $P_\X$ denote a measure on $\X$ and suppose we have a query budget of $n$.  Let $h_1$ denote the hypothesis selected by modified SGBS using $n/3$ of the queries and let $h_2$ denote the hypothesis selected by ERM from $n/3$ queries drawn independently from $P_\X$.  Draw the remaining $n/3$ queries independently from $P_\Delta$, the restriction of $P_\X$ to the set $\Delta$ on which $h_1$ and $h_2$ disagree, and let $\widehat{R}_\Delta(h_1)$ and $\widehat{R}_\Delta(h_2)$ denote the average number of errors made by $h_1$ and $h_2$ on these queries.  Select $\widehat{h} = \arg \min\{\widehat{R}_\Delta(h_1),\widehat{R}_\Delta(h_2)\}$.  Then, in general,
\begin{eqnarray*}
\E [R(\widehat h)]  & \leq  & \min\{\E[R(h_1)],\E[R(h_2)]\} \, + \, \sqrt{3/n} \ ,
\end{eqnarray*} 
where $R(h)$, $h\in \H$, denotes the probability of error of $h$ with respect to $P_\X$.
Furthermore, if the assumptions of
Theorem~\ref{thm3} hold and $h^* \in \H$, then
$$\P(\widehat h \neq h^*)  \ \leq \ Ne^{-\lambda n/3} \ + \ 2 e^{-n|1-2\alpha|^2/6} \ ,$$
where $\alpha$ is the noise bound.}
\end{theorem}

Note that if the assumptions of Theorem~\ref{thm3} hold, then the agnostic procedure performs almost as well as modified SGBS alone. In particular, the number of queries required to ensure that $\P(\widehat h \neq h^*) \leq \delta$ is proportional to  $\log\frac{\N}{\delta}$; optimal up to constant factors.  Also observe that the probability bound implies the following bound in expectation:
\begin{eqnarray*}
\E [R(\widehat h)]  & \leq  & R(h^*) + \  Ne^{-\lambda n/3} \ + \ 2 e^{-n|1-2\alpha|^2/6} \\
& \leq & R(h^*) \ + \ \N e^{-C n} \ ,
\end{eqnarray*} 
where $C>0$ is a constant depending on $\lambda$ and $\alpha$.  The exponential convergence of the expected risk is much faster than the usual parametric rate for ERM.

If the conditions of Theorem~\ref{thm3} are not met, then modified SGBS (alone) may perform poorly since it might select inappropriate queries and could even terminate with an incorrect hypothesis.  However, the expected error of the agnostic selection $\widehat h$ is within $\sqrt{3/n}$ of the expected error of ERM, with no assumptions on the underlying distributions. Note that the expected error of ERM is proportional to $n^{-1/2}$ in the worst-case situation.  Therefore, the agnostic procedure is near-optimal in general. The agnostic procedure offers this safeguard on performance.  The same approach could be used to derive agnostic procedures from any {\em active learning} scheme (i.e., learning from adaptively selected queries), including GBS or NGBS.   We also note that the important element in the agnostic procedure is the selection of queries; the proposed selection of $\widehat h$ based on those queries is convenient for proving the bounds, but not necessarily optimal.

\section{Applications of GBS}
\label{coherence}

In this section we examine a variety of common situations in which the neighborliness condition can be verified.  We will confine the discussion to GBS in the noise-free situation, and analogous results hold in the presence of noise.
For a given pair $(\X,\H)$, the effectiveness of GBS hinges on determining (or bounding) $\ch$ and establishing that $(\X,\H)$ are neighborly.  Recall the definition of the bound $\ch$ from (\ref{cstar}) and that $\N = |\H|$, the cardinality of $\H$.  A {\em trivial} bound for $\ch$ is
$$\max_{h\in \H} \left|\sum_{A\in \A} h(A) \, P(A)\right| \ \leq \ 
1-\N^{-1} \ ,$$ which is achieved by allocating $\N^{-1}$ mass to each hypothesis and evenly distributing $\frac{1}{2\N}$ mass to set(s) $A$ where $h(A)=+1$ and $\frac{1}{2\N}$ on set(s) $A$ where $h(A)=-1$.  Non-trivial coherence bounds are those for which there exists a $P$ and  $0 \leq c < 1$ that does not depend on $\N$ such that
$$\max_{h\in \H} \left|\sum_{A \in \A} h(A) \, P(A)\right| \ \leq \ c \ .$$
The coherence parameter $\ch$ is analytically determined or bounded in several illustrative applications below.  We also note that it may be known a priori that $\ch$ is bounded far way from $1$.  Suppose that for a certain $P$ on $\X$ (or $\A$) the absolute value of the first-moment of the correct hypothesis (w.r.t.\ $P$) is known to be upper bounded by a constant $c<1$.  Then all hypotheses that violate the bound can be eliminated from consideration.  Thus the constant $c$ is an upper bound on $\ch$.  Situations like this can arise, for example, in binary classification problems with side/prior knowledge that the marginal probabilities of the two classes are somewhat balanced.  Then the moment of the correct hypothesis, with respect to the marginal probability distribution on $\X$, is bounded far away from $1$ and $-1$.

The neighborly condition can be numerically verified in a straightforward fashion.  Enumerate the equivalence sets in $\A$ as $A_1,\dots,A_M$. To check whether the $k$-neighborhood graph is connected, form an $M \times M$ matrix $R_k$ whose $i,j$ entry is $1$ if $A_i$ and $A_j$ are k-neighbors and $0$ otherwise.  Normalize the rows of $R_k$ so that each sums to $1$ and denote the resulting {\em stochastic} matrix by $Q_k$.  The $k$-neighborhood graph is connected if and only if there exists an integer $\ell$, $0 < \ell \leq M$, such that $Q_k^\ell$ contains no zero entries.  This follows from the standard condition for state accessibility in Markov chains (for background on Markov chains one can refer to the textbook by Br\'{e}maud \cite{bremaud}).  The smallest $k$ for which the $k$-neighborhood graph is connected can be determined using a binary search over the set $1,\dots,M$, checking the condition above for each value of $k$ in the search.  This idea was suggested to me by Clayton Scott. Thus, the neighborly condition can be verified in polynomial-time in $M$.  Alternatively, in many cases the neighborly condition can be verified analytically, as demonstrated in following applications.

\subsection{One Dimensional Problems}
\label{cbs}
First we show that GBS reduces to classic binary search.  Let $\H = \{h_1,\dots,h_\N\}$ be the collection of binary-valued functions on $\X = [0,1]$ of the following form, $h_i(x):=\mbox{sign}\left(x-\frac{i}{N+1}\right)$ for $i=1,\dots,\N$ (and  $\mbox{sign}(0):=+1$).  Assume that $h^* \in \H$.

First consider the neighborly condition. Recall that $\A$ is the smallest partition of $\X$ into equivalence sets induced by $\H$.  In this case, each $A$ is an interval of the form $A_i = [\frac{i-1}{\N+1},\frac{i}{\N+1})$, $i=1,\dots,\N$.  Observe that only a single hypothesis, $h_i$, has different responses to queries from $A_i$ and $A_{i+1}$ and so they are $1$-neighbors, for $i=1,\dots,\N-1$.  Moreover, the $1$-neighborhood graph is connected in this case, and so $(\X,\H)$ is $1$-neighborly.

Next consider coherence parameter $\ch$.  Take $P$ to be two point masses at $x=0$ and $x=1$ of probability $1/2$ each.  Then $\left|\sum_{A\in \A} h(A) \, P(A)\right| \ = \ 0$ for every $h \in \H$, since $h(0) = -1$ and $h(1) = +1$.  Thus, $\ch = 0$.  Since $\ch= 0$ and $k=1$, we have $\alpha = 2/3$ and the query complexity of GBS is proportional to $\log \N$ according to Theorem~\ref{thm1}. The reduction factor of $2/3$, instead of $1/2$, arises because we allow the situation in which the number of hypotheses may be odd (e.g., given three hypotheses), the best query may eliminate just one).  If $\N$ is even, then the query complexity is $\log_2(\N)$, which is information-theoretically optimal.

Now let $\X = [0,1]$ and consider a finite collection of hypotheses $\H = \{h_i(x)\}_{i=1}^\N$, where $h_i$ takes the value $+1$ when $x \in [a_i,b_i)$, for a pair $0 \leq a_i < b_i \leq 1$, and $-1$ otherwise. Assume that $h^* \in \H$. The partition $\A$ again consists of intervals, and the neighborly condition is satisfied with $k=1$.  To bound $\ch$, note that the minimizing $P$ must place some mass within and outside each interval $[a_i,b_i)$.  If the intervals all have length at least $\ell>0$, then taking $P$ to be the uniform measure on $[0,1]$ yields that $\ch \leq 1-2\ell$, regardless of the number of interval hypotheses under consideration.  Therefore, in this setting Theorem~\ref{thm1} guarantees that GBS determines the correct hypothesis using at most a constant times $\log \N$ steps.

However, consider the special case in which the intervals are disjoint.  Then it is not hard to see that the best allocation of mass is to place $1/\N$ mass in each subinterval, resulting in $\ch =1-2\N^{-1}$.  And so, Theorem~\ref{thm1} only guarantees that GBS is will terminate in at most $\N$ steps (the number of steps required by exhaustive linear search). In fact, it is easy to see that no procedure can do better than linear search in this case and the query complexity of any method is proportional to $N$. However, note that if queries of a different form were allowed, then much better performance is possible.  For example, if queries in the form of dyadic subinterval tests were allowed (e.g., tests that indicate whether or not the correct hypothesis is $+1$-valued anywhere within a dyadic subinterval of choice), then the correct hypothesis can be identified through $\lceil \log_2 \N\rceil$ queries (essentially a binary encoding of the correct hypothesis).  This underscores the importance of the geometrical relationship between $\X$ and $\H$ embodied in the neighborly condition and the incoherence parameter $\ch$. Optimizing the query space to the structure of $\H$ is related to the notion of arbitrary queries examined in the work of Kulkarni et al \cite{kulkarni}, and   somewhat to the theory of compressed sensing developed by Candes et al \cite{candes:tit06a}  and Donoho \cite{donoho:tit06}.

\subsection{Multidimensional Problems}
\label{linear}

Let $\H=\{h_i\}_{i=1}^\N$ be a collection of multidimensional threshold functions of the following form. The threshold of each $h_i$ determined by (possibly nonlinear) decision surface in $d$-dimensional Euclidean space
and the queries are points in $\X:=\R^d$. It suffices to consider linear decision surfaces of the form 
\begin{eqnarray}
h_i(x) := \mbox{sign}(\langle a_i, x\rangle+b_i),
\label{surface}
\end{eqnarray}
where $a_i \in \R^d$, $\|a_i\|_2 = 1$, the offset $b_i\in \R$ satisfies $|b_i|\leq b$ for some constant $b<\infty$, and $\langle a_i, x\rangle$ denotes the inner product in $\R^d$.  Each hypothesis is associated with a halfspace of $\R^d$. Note that hypotheses of this form can be used to represent nonlinear decision surfaces, by first applying a mapping to an input space and then forming linear decision surfaces in the induced query space. The problem of learning multidimensional threshold functions arises commonly in computer vision (see the review of Swain and Stricker\cite{active_vision} and applications by Geman and Jedynak \cite{geman} and Arkin et al \cite{arkin:98}), image processing  studied by Korostelev and Kim \cite{kor1,kor2}, and active learning research; for example the investigations by Freund et al \cite{freund:97}, Dasgupta \cite{dasgupta:04}, Balcan et al \cite{nina:07}, and Castro and Nowak \cite{rui:08}.

First we show that the pair $(\R^d,\H)$ is $1$-neighborly. Each $A\in \A$ is a polytope in $\R^d$.  These polytopes are generated by intersections of the halfspaces corresponding to the hypotheses. Any two polytopes that share a common face are $1$-neighbors (the hypothesis whose decision boundary defines the face, and its complement if it exists, are the only ones that predict different values on these two sets).  Since the polytopes tessellate $\R^d$, the $1$-neighborhood graph of $\A$ is connected.

We next show that the the coherence parameter $\ch=0$.  Since the offsets of the hypotheses are all less than $b$ in magnitude, it follows that the distance from the origin to the nearest point of the decision surface of every hypothesis is at most $b$.  Let $P_r$ denote the uniform probability distribution on a ball of radius $r$ centered at the origin in $\R^d$.  Then for every $h$ of the form (\ref{surface})  there exists a constant $C>0$ (depending on $b$) such that
\begin{eqnarray*}
 \left|\sum_{A\in \A} h(A) \, P_r(A)\right| \ = \ \left|\int_{\R^d} h(x) \, dP_r(x)\right| &  \leq & \frac{C}{r^d}  \ ,
\end{eqnarray*}
and $\lim_{r\rightarrow
  \infty} \left|\int_{\R^d} h(x) \, dP_r(x)\right| = 0$.  Therefore $\ch = 0$, and it follows from Theorem~\ref{thm1} guarantees that GBS determines the correct multidimensional threshold in at most $\lceil \frac{\log 2}{\log(3/2)} \, \log_2 \N\rceil$ steps.  To the best of our knowledge this is a new result in the theory of learning multidimensional threshold functions, although similar query complexity bounds have been established for the subclass of linear threshold functions with $b=0$ (threshold boundaries passing through the origin); see for example the work of Balcan et al \cite{nina:07}.  These results are
based on somewhat different learning algorithms, assumptions and analysis techniques.

Observe that if $\H$ is an $\epsilon$-dense (with respect Lesbegue measure over a compact set in $\R^d$) subset of the continuous class of threshold functions of the form (\ref{surface}), then the size of the $\H$ satisfies $\log |\H| = \log\N \propto d \log \epsilon^{-1}$. Therefore the query complexity of GBS is proportional to the metric entropy of the continuous class, and it follows from the results of Kulkarni et al \cite{kulkarni} that no learning algorithm exists with a lower query complexity (up to constant factors).  Furthermore, note that the computational complexity of GBS for hypotheses of the form (\ref{surface}) is proportional to the cardinality of $\A$, which is equal to the number of polytopes generated by intersections of half-spaces. It is a well known fact (see Buck \cite{buck:43}) that $|\A| = \sum_{j=0}^d {\N\choose{j}} \propto \N^d$.  Therefore, GBS is a polynomial-time algorithm for this problem.  In general, the cardinality of $\A$ could be as large as $2^\N$.

Next let $\H$ again be the hypotheses of the form (\ref{surface}), but let $\X := [-1,1]^d$, instead of all of $\R^d$.  This constraint on the query space affects the bound on the coherence $\ch$.  To bound $\ch$, let $P$ be point masses of probability $2^{-d}$ at each of the $2^d$ vertices of the cube $[-1,1]^d$ (the natural generalization of the $P$ chosen in the case of classic binary search in Section~\ref{cbs} above).  Then $\left|\sum_{A\in \A} h(A) \, P(A)\right| = \left|\int_\X h(x) \, dP(x)\right| \ \leq \ 1-2^{-d+1}$ for every $h \in \H$, since for each $h$ there is at least one vertex on where it predicts $+1$ and one where it predicts $-1$. Thus, $\ch \leq 1-2^{-d+1}$.  We conclude that the GBS determines the correct hypothesis in proportional to $2^{d} \log \N$ steps.  The dependence on $2^d$ is unavoidable, since it may be that each threshold function takes that value $+1$ only at one of the $2^d$ vertices and so each vertex must be queried. A noteworthy special case is arises when $b=0$ (i.e., the threshold boundaries pass through the origin).  In this case, with $P$ as specified above, $\ch = 0$, since each hypothesis responds with $+1$ at half of the vertices and $-1$ on the other half.  Therefore, the query complexity of GBS is at most $\lceil \frac{\log 2}{\log(3/2)} \, \log_2 \N\rceil$, independent of the dimension.  As discussed above, similar results for this special case have been previously reported based on different algorithms and analyses; see the results in the work of Balcan et al \cite{nina:07} and the references therein. Note that even if the threshold boundaries do not pass through the origin, and therefore the number of queries needed is proportional to $\log \N$ so long as $|b|<1$.  The dependence on dimension $d$ can also be eliminated if it is known that for a certain distribution $P$ on $\X$ the absolute value of the moment of the correct hypothesis w.r.t.\ $P$ is known to be upper bounded by a constant $c<1$, as discussed at the beginning of this section.


Finally, we also mention hypotheses associated with axis-aligned rectangles in $[0,1]^d$, the multidimensional version of the interval hypotheses considered above.  An axis-aligned rectangle is defined by its boundary coordinates in each dimension, $\{a_j,b_j\}_{j=1}^d$, $0 \leq a_j < b_j\leq 1$.  The hypothesis associated with such a rectangle takes the value $+ 1$ on the set $\{x \in [0,1]^d: \ a_j \leq x_j \leq b_j, \ j=1,\dots,d\}$ and $-1$ otherwise.  The complementary hypothesis may also be included.  Consider a finite collection $\H$ of hypotheses of this form.  If the rectangles associated with each $h\in \H$ have volume at least $\nu$, then by taking $P$ to be the uniform measure on $[0,1]^d$ it follows that the coherence parameter $\ch \leq 1-2\nu$ for this problem.  The cells of partition $\A$ of $[0,1]$ associated with a collection of such hypotheses are rectangles themselves.  If the boundaries of the rectangles associated with the hypotheses are distinct, then the $1$-neighborhood graph of $\A$ is connected.  Theorem~\ref{thm1} implies that the number of queries needed by GBS to determine the
correct rectangle is proportional to $\log \N/\log((1-\nu)^{-1})$.

\subsection{Discrete Query Spaces}

In many situations both the hypothesis and query spaces may be discrete.  A machine learning application, for example, may have access to a large (but finite) pool of unlabeled examples, any of which may be queried for a label.  Because obtaining labels can be costly, ``active'' learning algorithms select only those examples that are predicted to be highly informative for labeling.  Theorem~\ref{thm1} applies equally well to continuous or discrete query spaces.  For example, consider the linear separator case, but instead of the query space $\R^d$ suppose that $\X$ is a finite subset of points in $\R^d$.  The hypotheses again induce a partition of $\X$ into subsets $\A(\X,\H)$, but the number of subsets in the partition may be less than the number in $\A(\R^d,\H)$.  Consequently, the neighborhood graph of $\A(\X,\H)$ depends on the specific points that are included in $\X$ and may or may not be connected.  As discussed at the beginning of this section, the neighborly condition can be verified in polynomial-time (polynomial in $|\A| \leq |\X|$).  

Consider two illustrative examples.  Let $\H$ be a collection of linear separators as in (\ref{surface}) above and first reconsider the partition $\A(\R^d,\H)$.  Recall that each set in $\A(\R^d,\H)$ is a polytope.  Suppose that a discrete set $\X$ contains at least one point inside each of the polytopes in $\A(\R^d,\H)$.  Then it follows from the results above that $(\X,\H)$ is $1$-neighborly.  Second, consider a simple case in $d=2$ dimensions.  Suppose $\X$ consists of just three non-colinear points $\{x_1,x_2,x_3\}$ and suppose that $\H$ is comprised of six classifiers, $\{h_1^+,h_1^-,h_2^+,h_2^-,h_3^+,h_3^-\}$, satisfying $h_i^+(x_i)=+1$, $h_i^+(x_j)=-1, j\neq i$ , $i=1,2,3$, and $h_i^-=-h_i^+$, $i=1,2,3$.  In this case, $\A(\X,\H) = \left\{\{x_1\},\{x_2\},\{x_3\}\right\}$ and the responses to any pair of queries differ for four of the six hypotheses. Thus, the $4$-neighborhood graph of $\A(\X,\H)$ is connected, but the $1$-neighborhood is not.

Also note that a finite query space naturally limits the number of hypotheses that need be considered.  Consider an uncountable collection of hypotheses. The number of unique labeling assignments generated by these hypotheses can be bounded in terms of the VC dimension of the class; see the book by Vapnik for more information on VC theory \cite{vapnik95}.  As a result, it suffices to consider a finite subset of the hypotheses consisting of just one representative of each unique labeling assignment.  Furthermore, the computational complexity of GBS is proportional to $\N \, |\X|$ in this case.

\section{Related Work}

Generalized binary search can be viewed as a generalization of classic binary search, Shannon-Fano coding as noted by Goodman and Smyth \cite{goodman:88}, and channel coding with noiseless feedback as studied by Horstein \cite{horstein:63}.  Problems of this nature arise in many applications, including channel coding (e.g., the work of Horstein \cite{horstein:63} and Zigangirov \cite{z2}), experimental design (e.g., as studied by R{\'e}nyi   \cite{renyi:61,renyi:65}), disease diagnosis (e.g., see the  work of Loveland \cite{tree2}), fault-tolerant computing (e.g., the work of Feige et al \cite{feige:94}), the scheduling problem considered by  Kosaraju et al \cite{tree1}, computer vision problems investigated by Geman and Jedynak \cite{geman} and Arkin et al \cite{arkin:98}), image processing problems studied by Korostelev and Kim \cite{kor1,kor2}, and active learning research; for example the investigations by Freund et al \cite{freund:97}, Dasgupta \cite{dasgupta:04}, Balcan et al \cite{nina:07}, and Castro and Nowak \cite{rui:08}. 

Past work has provided a partial characterization of this problem. If the responses to queries are noiseless, then selecting the sequence of queries from $\X$ is equivalent to determining a binary decision tree, where a sequence of queries defines a path from the root of the tree (corresponding to $\H$) to a leaf (corresponding to a single element of $\H$).  In general the determination of the optimal (worst- or average-case) tree is NP-complete as shown by Hyafil and Rivest \cite{tree4}.  However, there exists a greedy procedure that yields query sequences that are within a factor of $\log \N$  of the optimal search tree depth; this result has been discovered independently by several researchers including Loveland \cite{tree2}, Garey and Graham \cite{tree2}, Arkin et al \cite{tree3}, and Dasgupta \cite{dasgupta:04}. The greedy procedure is referred to here as {\em
  Generalized Binary Search} (GBS) or the {\em splitting algorithm}, and it reduces to classic binary search, as discussed in Section~\ref{cbs}.

The number of queries an algorithm requires to determine $h^*$ is called the {\em query complexity} of the algorithm. Since the hypotheses are assumed to be distinct, it is clear that the query complexity of GBS is at most $\N$ (because it is always possible to find query that eliminates at least one hypothesis at each step).  In fact, there are simple examples (see Section~\ref{cbs}) demonstrating that this is the best one can hope to do in general.  However, it is also true that in many cases the performance of GBS can be much better, requiring as few as $\log_2(\N)$ queries.  In classic binary search, for example, half of the hypotheses are eliminated at each step (e.g., refer to the textbook by Cormen et al \cite{cormen:09}).
R{\'e}nyi first considered a form of binary search with noise \cite{renyi:61} and explored its connections with information theory \cite{renyi:65}.  In particular, the problem of sequential transmission over a binary symmetric channel with noiseless feedback, as formulated by Horstein \cite{horstein:63} and studied by Burnashev and Zigangirov \cite{bz} and more recently by Pelc et al \cite{pelc:02}, is equivalent to a noisy binary search problem. 

There is a large literature on learning from queries; see the review articles by Angluin \cite{angluin:88,angluin:01}.  This paper focuses exclusively on membership queries (i.e., an $x \in \X$ is the query and the response is $h^*(x)$), although other types of queries (equivalence, subset, superset, disjointness, and exhaustiveness) are possible as discussed by Angluin \cite{angluin:88}.  
{\em Arbitrary queries} have also been investigated, in which the query is a subset of $\H$ and the output is $+1$ if $h^*$ belongs to the subset and $-1$ otherwise.  A finite collection of hypotheses $\H$ can be successively halved using arbitrary queries, and so it is possible to determine $h^*$ with $\log_2\N$ arbitrary queries, the information-theoretically optimal query complexity discussed by Kulkarni et al \cite{kulkarni}.  Membership queries are the most natural in function learning problems, and because this paper deals only with this type we will simply refer to them as queries throughout the rest of the paper.  The number of queries required to determine a binary-valued function in a finite collection of hypotheses can be bounded (above and below) in terms of a combinatorial parameter of $(\X,\H)$ due to Heged{\"u}s\cite{hegedus:95} (see the work of Hellerstein et al \cite{hellerstein:96} for related work).  Due to its combinatorial nature,  computing such bounds are generally NP-hard.  In contrast, the geometric relationship between $\X$ and $\H$ developed in this paper leads to an upper bound on the query complexity that can be determined analytically or computed in polynomial time in many cases of interest.

The term GBS is used in this paper to emphasize connections and similarities with classic binary search, which is a special case the general problem considered here.  Classic binary search is equivalent to learning a one-dimensional binary-valued threshold function by selecting point evaluations of the function according to a bisection procedure.  Consider the threshold function $h_t(x):=\mbox{sign}(x-t)$ on the interval $\X:=[0,1]$ for some threshold value $t \in (0,1)$.  Throughout the paper
we adopt the convention that $\mbox{sign}(0)=+1$. Suppose that $t$ belongs to the discrete set $\{\frac{1}{N+1},\dots,\frac{N}{N+1}\}$ and let $\H$ denote the collection of threshold functions $h_1,\dots,h_{N}$. The value of $t$ can then determined from a constant times $\log N$ queries using a bisection procedure analogous to the game of twenty questions. In fact, this is precisely what GBS performs in this case (i.e., GBS reduces to classic binary search in this setting). If $N = 2^m$ for some integer $m$, then each point evaluation provides one bit in the $m$-bit binary expansion of $t$.  Thus, classic binary search is information-theoretically optimal; see the book by Traub, Wasilkowski and Wozniakowski \cite{traub88} for a nice treatment of classic bisection and binary search.  The main results of this paper generalize the salient aspects of classic binary search to a much broader class of problems.
In many (if not most) applications it is unrealistic to assume that the responses to queries are without error.  A form of binary search with noise appears to have been first posed by R{\'e}nyi \cite{renyi:61}.  The noisy binary search problem arises in sequential transmission over a binary symmetric channel with noiseless feedback studied by Horstein \cite{horstein:63} and Zigangirov \cite{z1,z2}. The survey paper by Pelc et al \cite{pelc:02} discusses the connections between search and coding problems. In channel coding with feedback, each threshold corresponds to a unique binary codeword (the binary expansion of $t$).  Thus, channel coding with noiseless feedback is equivalent to the problem of learning a one-dimensional threshold function in binary noise, as noted by Burnashev and Zigangirov \cite{bz}.  The  near-optimal solutions to the {\em noisy binary search} problem first appear in these two contexts.  Discrete versions of Horstein's probabilistic bisection procedure \cite{horstein:63} were shown to be information-theoretically optimal (optimal decay of the error probability) in the works of Zigangirov and Burnashev \cite{z1,z2,bz}.  More recently, the same procedure was independently proposed and analyzed in the context of noise-tolerant versions of the classic binary search problem by Karp and Kleinberg \cite{kk:07}, which was motivated by applications ranging from investment planning to admission control in queueing networks. Closely related approaches are considered in the work of Feige et al \cite{feige:94 }. The noisy binary search problem has found important applications in the minimax theory of sequential, adaptive sampling procedures proposed by Korostelov and Kim \cite{kor1,kor2} for image recovery and binary classification problems studied by Castro and Nowak \cite{rui:08}.  We also mention the works of Rivest et al  \cite{rivest}, Spencer \cite{spencer} and Aslam and Dhagat \cite{dhagat1}, and Dhagat et al \cite{dhagat}, which consider adversarial situations in which the total number of erroneous oracle responses is fixed in advance.

One straightforward approach to noisy GBS is to follow the GBS algorithm, but to repeat the query at each step multiple times in order to decide whether the response is more probably $+1$ or $-1$.   This simple approach has been studied in the context of noisy versions of classic binary search and shown to perform significantly worse than other approaches in the work of Karp and Kleinberg \cite{kk:07}; perhaps not surprising since this is essentially a simple repetition code approach to communicating over a noisy channel.  A near-optimal noise-tolerant version of GBS was developed in this paper.  The algorithm can be viewed as a non-trivial generalization of Horstein's probabilistic bisection procedure.  Horstein's method relies on the special structure of classic binary search, namely that the hypotheses and queries can be naturally ordered together in the unit interval.  Horstein's method is a sequential Bayesian procedure.  It begins with uniform distribution over the set of hypotheses.  At each step, it  queries at the point that bisects the probability mass of the current distribution over hypotheses, and then updates the distribution according to Bayes rule. Horstein's procedure isn't directly applicable to situations in which the hypotheses and queries cannot be ordered togetherl, but the geometric condition developed in this paper provides similar structure that is exploited here to devise a generalized probabilistic bisection procedure. The key elements of the procedure and the analysis of its convergence are fundamentally different from those in the classic binary search work of Burnashev and Zigangirov \cite{bz} and Karp and Kleinberg \cite{kk:07}.

\section{Conclusions and Possible Extensions}
\label{sec:conclude}
This paper investigated a generalization of classic binary search, called GBS, that extends it to arbitrary query and hypothesis spaces.  While the  GBS algorithm
is well-known, past work has only partially characterized its capabilities.  
This paper developed new conditions under which GBS (and a noise-tolerant variant) achieve the information-theoretically optimal query complexity.  The new conditions are based on a novel geometric relation between the query and hypothesis spaces, which is verifiable analytically and/or computationally in many cases of practical interest.  The main results are applied to learning multidimensional threshold functions, a problem arising routinely in image processing and machine learning.

Let us briefly consider some possible extensions and open problems.  First recall that in noisy situations it is assumed that the binary noise probability has a known upper bound $\alpha < 1/2$.  It is possible to accommodate situations in which the bound is unknown a priori.  This can be accomplished using an NGBS algorithm in which the number of repetitions of each query, $R$, is determined adaptively to adjust to the unknown noise level.  This procedure was developed by the author in \cite{nowak_gbs}, and is based on a straightforward, iterated application of Chernoff's bound. Similar strategies have been suggested as a general approach for devising noise-tolerant learning algorithms \cite{matti}.  Using an adaptive procedure for adjusting the number of repetitions of each query yields an NGBS algorithm with query complexity bound proportional to $\log \N \log \frac{\log \N}{\delta}$, the same order as that of the NGBS algorithm discussed above which assumed a known bound $\alpha$.  Whether or not the additional logarithmic factor can be removed if the noise bound $\alpha$ is unknown is an open question.
Adversarial noise models in which total number of errors is fixed in advance, like those considered by  Rivest et al  \cite{rivest} and Spencer \cite{spencer}, are also of interest in classic binary search problems.  Repeating each query multiple times and taking the majority vote of the responses, as in the NGBS algorithm, is a standard approach to adversarial noise.   Thus, NGBS provides an algorithm for generalized binary search with adversarial noise.

Finally, we suggest that the salient features of the GBS algorithms could be extended to handle continuous, uncountable classes of hypotheses.  For example, consider the continuous class of halfspace threshold functions on $\R^d$.  This class is indexed parametrically and it is possible to associate a volume measure with the class (and subsets of it) by introducing a measure over the parameter space.  At each step of a GBS-style algorithm, all inconsistent hypotheses are eliminated and the next query is selected to split the volume of the parameter space corresponding to the remaining hypotheses, mimicking the splitting criterion of the GBS algorithm presented here.
\\ \ \\
\noindent {\sc Acknowledgements}. The author thanks R.\ Castro, A.\ Gupta, C.\ Scott, A.\ Singh and the anonymous reviewers for helpful feedback and suggestions.

\section{Appendix}

\subsection{Proof of Lemma~\ref{martingale}}

First we derive the precise form of $p_1,p_2,\dots$ is derived as follows.  Let
$\delta_i = (1+\sum_h p_i(h) \, z_i(h))/2$, the weighted proportion of
hypotheses that agree with $y_i$.  The factor that normalizes the
updated distribution in (\ref{update1}) is related to $\delta_i$ as
follows.  Note that $\sum_h p_i(h) \, \beta^{(1-z_i(h))/2}
(1-\beta)^{(1+z_i(h))/2} = \sum_{h:z_i(h)=-1} p_i(h)\beta +
\sum_{h:z_i(h)=1} p_i(h)(1-\beta) = \mbox{$(1-\delta_i)\beta+\delta_i
  (1-\beta)$}$. Thus,
\begin{eqnarray*}p_{i+1}(h) \ = \ p_i(h) \, \frac{\beta^{(1-z_i(h))/2}
  (1-\beta)^{(1+z_i(h))/2}}{(1-\delta_i)\beta+\delta_i (1-\beta)}
\end{eqnarray*}
Denote the reciprocal of the update factor for $p_{i+1}(h^*)$ by
\begin{eqnarray}
\gamma_i := \frac{(1-\delta_i)\beta+\delta_i
  (1-\beta)}{\beta^{(1-z_i(h^*))/2} (1-\beta)^{(1+z_i(h^*))/2}} \ ,
\label{gamma}
\end{eqnarray}
where $z_i(h^*) = h^*(x_i)y_i$, and observe that $p_{i+1}(h^*) =
p_i(h^*)/\gamma_i$.  Thus,
$$\frac{\mm_{i+1}}{\mm_i} \ = \ \frac{(1- p_i(h^*)/\gamma_i)p_i(h^*)}{
  p_i(h^*)/\gamma_i(1-p_i(h^*))} \ = \ \frac{\gamma_i- p_i(h^*)}{
  1-p_i(h^*)} \ .$$
We prove that
$\mathbb{E}[{\mm_{i+1}}|p_i]\leq\mm_i$ by showing  that
$\mathbb{E}[\gamma_i|p_i] \leq 1$.  To accomplish this, we will let $p_i$ be arbitrary.

For every $A\in {\cal A}$ and every $h \in \H$ let $h(A)$ denote the
value of $h$ on the set $A$.  Define $\delta_{A}^+ = (1+\sum_h
  p_i(h)h(A))/2$, the proportion of hypotheses that take the value
$+1$ on $A$.  Let $A_i$ denote that set that $x_i$ is selected
from, and consider the four possible situations:
\begin{eqnarray*}
h^*(x_i) = +1, \, y_i = +1: & \gamma_i =\frac{(1-\delta_{A_i}^+) \beta + \delta_{A_i}^+ (1-\beta)}{1-\beta} \\
h^*(x_i) = +1, \, y_i = -1: & \gamma_i =\frac{\delta_{A_i}^+ \beta + (1-\delta_{A_i}^+) (1-\beta)}{\beta} \\
h^*(x_i) = -1,\,  y_i = +1: & \gamma_i =\frac{(1-\delta_{A_i}^+) \beta + \delta_{A_i}^+ (1-\beta)}{\beta} \\
h^*(x_i) = -1, \, y_i = -1: & \gamma_i =\frac{\delta_{A_i}^+ \beta + (1-\delta_{A_i}^+) (1-\beta)}{1-\beta}
\end{eqnarray*}
To bound $\E[\gamma_i|p_i]$ it is helpful to condition on $A_i$.
Define $q_i := \P(y_i\neq h^*(x_i))$.
If $h^*(A_i)=+1$, then
\begin{eqnarray*}
  \E[\gamma_i|p_i,A_i] 
  & = &  \frac{(1-\delta_{A_i}^+)\beta+\delta_{A_i}^+ (1-\beta)}{1-\beta}(1-q_i)\ + \ \frac{\delta_{A_i}^+\beta+(1-\delta_{A_i}^+ )(1-\beta)}{\beta}q_i 
 \\
& = & \delta_{A_i}^+ + (1-\delta_{A_i}^+)\left[\frac{\beta(1-q_i)}{1-\beta}+\frac{q_i(1-\beta)}{\beta} \right] \ .
\end{eqnarray*}
Define $\gamma_i^+(A_i):=\delta_{A_i}^+ + (1-\delta_{A_i}^+)\left[\frac{\beta(1-q_i)}{1-\beta}+\frac{q_i(1-\beta)}{\beta} \right]$.
Similarly, if $h^*(A_i)=-1$, then
\begin{eqnarray*}
\E[\gamma_i|p_i,A_i] 
& = & (1- \delta_{A_i}^+) + \delta_{A_i}^+ \left[\frac{\beta(1-q_i)}{1-\beta}+\frac{q_i(1-\beta)}{\beta} \right] \ =: \ \gamma_i^{-}(A_i) 
\end{eqnarray*}
By assumption $q_i \leq \alpha < 1/2$, and since $\alpha \leq \beta <
1/2$ the factor
$\frac{\beta(1-q_i)}{1-\beta}+\frac{q_i(1-\beta)}{\beta} \leq
\frac{\beta(1-\alpha)}{1-\beta}+\frac{\alpha(1-\beta)}{\beta} \leq 1$
(strictly less than $1$ if $\beta >\alpha$).
Define
\begin{eqnarray*}
\varepsilon_0 :=
1-\frac{\beta(1-\alpha)}{1-\beta}-\frac{\alpha(1-\beta)}{\beta} \ ,
\end{eqnarray*}
to obtain the bounds
\begin{eqnarray}
\gamma_i^+(A_i) & \leq &  \delta_{A_i}^+ + (1-\delta_{A_i}^+)(1-\varepsilon_0) \label{pb} \ , \\
\gamma_i^-(A_i) & \leq &  \delta_{A_i}^+ (1-\varepsilon_0) + (1- \delta_{A_i}^+) \label{nb} \ .
\end{eqnarray}
Observe that for every $A$ we have $0 < \delta_{A}^+ < 1$, since at least one hypothesis takes the value $-1$ on $A$ and $p(h)>0$ for all $h\in \H$.  Therefore both $\gamma_i^+(A_i)$ and $\gamma_i^-(A_i)$ are less or equal to $1$, and it follows that $\mathbb{E}[\gamma_i|p_i]\leq1$ (and strictly less than $1$ if $\beta > \alpha$).
 \hfill $\blacksquare$

\subsection{Proof of Theorem~\ref{thm3}}
The proof amounts to obtaining upper bounds for $\gamma_i^+(A_i)$ and
$\gamma_i^-(A_i)$, defined above in (\ref{pb}) and (\ref{nb}). 
Consider two distinct situations.  Define $b_i:=\min_{A \in {\cal
    A}} |W(p_i,A)|$.  First suppose that there do not exist
neighboring sets $A$ and $A'$ with $W(p_i,A) > b_i$ and $W(p_i,A') <
-b_i$.  Then by Lemma~\ref{lemma1}, this implies that $b_i \leq \ch$,
and according the query selection step of the modified SGBS algorithm,
$A_i = \arg \min_{A} \left| W(p_i,A)\right|$.  Note that because
$|W(p_i,A_i)|\leq \ch$, $(1-\ch)/2 \ \leq \ \delta_{A_i}^+ \ \leq
(1+\ch)/2$.  Hence, both $\gamma_i^+(A_i)$ and $\gamma_i^-(A_i)$ are
bounded above by $1-\varepsilon_0(1-\ch)/2$.

Now suppose that there exist neighboring sets $A$ and $A'$ with
$W(p_i,A) > b_i$ and $W(p_i,A') < -b_i$.   
Recall that in this case $A_i$ is randomly chosen to be $A$ or $A'$
with equal probability.  Note that $\delta_{A}^+ > (1+b_i)/2$ and
$\delta_{A'}^+ < (1-b_i)/2$.  
If $h^*(A) = h^*(A') = +1$, then applying (\ref{pb}) results in
\begin{eqnarray*}
\E[\gamma_i|p_i,A_i \in \{A,A'\}]  & < & \frac{1}{2}(1 + \frac{1-b_i}{2}+\frac{1+b_i}{2}(1-\varepsilon_0)) \ = \ \frac{1}{2}(2-\varepsilon_0\frac{1+b_i}{2})  \ \leq \ 1-\varepsilon_0/4 \ ,
\end{eqnarray*}
since $b_i>0$.  Similarly, if $h^*(A) = h^*(A') = -1$, then (\ref{nb})
yields $\E[\gamma_i|p_i,A_i \in \{A,A'\}] < 1-\varepsilon_0/4$.
If $h^*(A) = -1$ on $A$ and $h^*(A') = +1$, then applying (\ref{nb}) on
$A$ and (\ref{pb}) on $A'$ yields
\begin{eqnarray*}
\E[\gamma_i|p_i,A_i \in \{A,A'\}] &  \leq & \frac{1}{2}\left(\delta_{A}^+ (1-\varepsilon_0) + (1- \delta_{A}^+) +\delta_{A'}^+ + (1-\delta_{A'}^+)(1-\varepsilon_0)\right) \\
& = &  \frac{1}{2}(1-\delta_A^+ +\delta_{A'}^+ + (1-\varepsilon_0)(1+\delta_A^+ -\delta_{A'}^+))  \\
& = &  \frac{1}{2}(2-\varepsilon_0(1+\delta_A^+-\delta_{A'}^+)) \\ & = & 1-\frac{\varepsilon_0}{2} (1+\delta_A^+-\delta_{A'}^+)  \ \leq \ 1-\varepsilon_0/2 \ ,
\end{eqnarray*}
since $0 \leq \delta_A^+-\delta_{A'}^+ \leq 1$.
The final possibility is that $h^*(A) = +1$ and $h^*(A') = -1$.  Apply
(\ref{pb}) on $A$ and (\ref{nb}) on $A'$ to obtain
\begin{eqnarray*}
\E[\gamma_i|p_i,A_i \in \{A,A'\}] & \leq &\hspace{-.1in} \frac{1}{2}\left(\delta_A^+ + (1-\delta_A^+)(1-\varepsilon_0)+\delta_{A'}^+ (1-\varepsilon_0) + (1- \delta_{A'}^+)\right) \\
 & = &  \frac{1}{2}(1+\delta_A^+-\delta_{A'}^+ + (1-\varepsilon_0)(1-\delta_A^++\delta_{A'}^+))
\end{eqnarray*}
Next, use the fact that because $A$ and $A'$ are neighbors,
$\delta_A^+ -\delta_{A'}^+ = p_i(h^*)-p_i(-h^*)$; if $-h^*$ does not
belong to $\H$, then $p_i(-h^*) = 0$.  Hence,
\begin{eqnarray*}
\E[\gamma_i|p_i,A_i \in \{A,A'\}] 
 & \leq &  \frac{1}{2}(1+\delta_A^+-\delta_{A'}^+ + (1-\epsilon_0)(1-\delta_A^++\delta_{A'}^+)) \\
& = &  \frac{1}{2}(1+p_i(h^*)-p_i(-h^*) + (1-\epsilon_0)(1-p_i(h^*)+p_i(-h^*))) \\
& \leq &  \frac{1}{2}(1+p_i(h^*) + (1-\epsilon_0)(1-p_i(h^*))) \ = \  1-\frac{\varepsilon_0}{2}(1-p_i(h^*))\ ,
\end{eqnarray*}
since the bound is maximized when $p_i(-h^*)=0$. 
Now bound $\E[\gamma_i|p_i]$ by the maximum of
the conditional bounds above to obtain
\begin{eqnarray*}
\E[\gamma_i|p_i] &  \leq & \max
\left\{1-\frac{\varepsilon_0}{2}(1-p_i(h^*)) \, , \, 1-\frac{\varepsilon_0}{4} \, , \,
    1-(1-\ch)\frac{\varepsilon_0}{2} \right\} \ ,
\end{eqnarray*}
and thus it is easy to see that
\begin{eqnarray*}
\E\left[\frac{\mm_{i+1}}{\mm_i}|p_i\right] & = & \frac{\E\left[
  \gamma_i|p_i\right]- p_i(h^*)}{ 1-p_i(h^*)} \ \leq \
1-\min\left\{\frac{\varepsilon_0}{2}(1-\ch),\frac{\varepsilon_0}{4}\right\} \ .
\hspace{.5in} \blacksquare \end{eqnarray*}

\subsection{Proof of Theorem~\ref{thm5}}
First consider the bound on $\E[R(\widehat h)]$.  Let $\delta_n = 2\, e^{-n|R_\Delta(h_1)-R_\Delta(h_2)|^2/6}$ and consider the conditional
expectation  $\E [R(\widehat h) |h_1,h_2] $; i.e., expectation with respect to the
$n/3$ queries drawn from the region $\Delta$, conditioned on the $2n/3$ queries
used to select $h_1$ and $h_2$. By Lemma~\ref{runoff} 
\begin{eqnarray*}
\E [R(\widehat h) |h_1,h_2]  & \leq  & (1-\delta_n) \, \min\{R(h_1),R(h_2)\}  \ + \  \delta_n \, \max\{R(h_1),R(h_2)\}  \ , \\
& = & \min\{R(h_1),R(h_2)\} \ + \ \delta_n \, \left[\max\{R(h_1),R(h_2)\}-\min\{R(h_1),R(h_2)\}\right]   \ , \\
& = & \min\{R(h_1),R(h_2)\} \ + \ \delta_n \, |R(h_1)-R(h_2)|    \ , \\
& = & \min\{R(h_1),R(h_2)\} \ + \ 2 \, |R(h_1)-R(h_2)|  \, e^{-n|R_\Delta(h_1)-R_\Delta(h_2)|^2/6}  \ ,  \\
& \leq & \min\{R(h_1),R(h_2)\}+ 2 |R(h_1)-R(h_2)|  \, e^{-n|R(h_1)-R(h_2)|^2/6}   \ , 
\end{eqnarray*}
where the last inequality follows from the fact that $|R(h_1)-R(h_2)| \leq |R_\Delta(h_1)-R_\Delta(h_2)|$.

The function $2u \, e^{-\zeta u^2}$ attains its maximum at $u=\sqrt{\frac{2}{e \zeta}} < \sqrt{1/\zeta}$, and
therefore
\begin{eqnarray*}
\E [R(\widehat h) |h_1,h_2]  & \leq  & \min\{R(h_1),R(h_2)\} \, + \, \sqrt{3/n} \ .
\end{eqnarray*}
Now taking the expectation with respect to $h_1$ and $h_2$ (i.e., with respect to the
queries used for the selection of $h_1$ and $h_2$)
\begin{eqnarray*}
\E [R(\widehat h)]  & \leq  & \E[\min\{R(h_1),R(h_2)\}] \, + \, \sqrt{3/n} \ , \\
& \leq  & \min\{\E[R(h_1)],\E[R(h_2)]\} \, + \, \sqrt{3/n} \ ,
\end{eqnarray*}
by Jensen's inequality.

Next consider the bound on $\P(h_1 \neq h^*) $.  This also follows from an application of Lemma~\ref{runoff}.  Note that if the conditions of Theorem~\ref{thm3} hold, then $\P(h_1 \neq h^*) \ \leq \ Ne^{-\lambda n/3}$.  Furthermore, if $h_1 = h^*$ and $h_2 \neq h^*$, then $|R_\Delta(h_1)-R_\Delta(h_2)|\geq |1-2\alpha|$. The bound on $\P(\widehat h \neq h^*)$ follows by applying the union bound to the events $h_1 = h^*$ and $R(\widehat{h}) > \min\{R(h_1),R(h_2)\}$. \hfill $\blacksquare$


\begin{thebibliography}{10}
\providecommand{\url}[1]{#1}
\csname url@samestyle\endcsname
\providecommand{\newblock}{\relax}
\providecommand{\bibinfo}[2]{#2}
\providecommand{\BIBentrySTDinterwordspacing}{\spaceskip=0pt\relax}
\providecommand{\BIBentryALTinterwordstretchfactor}{4}
\providecommand{\BIBentryALTinterwordspacing}{\spaceskip=\fontdimen2\font plus
\BIBentryALTinterwordstretchfactor\fontdimen3\font minus
  \fontdimen4\font\relax}
\providecommand{\BIBforeignlanguage}[2]{{%
\expandafter\ifx\csname l@#1\endcsname\relax
\typeout{** WARNING: IEEEtran.bst: No hyphenation pattern has been}%
\typeout{** loaded for the language `#1'. Using the pattern for}%
\typeout{** the default language instead.}%
\else
\language=\csname l@#1\endcsname
\fi
#2}}
\providecommand{\BIBdecl}{\relax}
\BIBdecl

\bibitem{goodman:88}
R.~M. Goodman and P.~Smyth, ``Decision tree design from a communication theory
  standpoint,'' \emph{IEEE Trans. Info. Theory}, vol.~34, no.~5, pp. 979--994,
  1988.

\bibitem{horstein:63}
M.~Horstein, ``Sequential decoding using noiseless feedback,'' \emph{IEEE
  Trans. Info. Theory}, vol.~9, no.~3, pp. 136--143, 1963.

\bibitem{bz}
M.~V. Burnashev and K.~S. Zigangirov, ``An interval estimation problem for
  controlled observations,'' \emph{Problems in Information Transmission},
  vol.~10, pp. 223--231, 1974.

\bibitem{z2}
K.~S. Zigangirov, ``Upper bounds for the error probability of feedback
  channels,'' \emph{Probl. Peredachi Inform.}, vol.~6, no.~2, pp. 87--82, 1970.

\bibitem{renyi:61}
A.~R{\'e}nyi, ``On a problem in information theory,'' \emph{MTA Mat. Kut. Int.
  Kozl.}, p. 505–516, 1961, reprinted in {\em Selected Papers of Alfred
  R{\'e}nyi}, vol.\ 2, P. Turan, ed., pp. 631-638. Akademiai Kiado, Budapest,
  1976.

\bibitem{renyi:65}
------, ``On the foundations information theory,'' \emph{Review of the
  International Statistical Institute}, vol.~33, no.~1, 1965.

\bibitem{tree2}
D.~W. Loveland, ``Performance bounds for binary testing with arbitrary
  weights,'' \emph{Acta Informatica}, vol.~22, pp. 101--114, 1985.

\bibitem{feige:94}
U.~Feige, E.~Raghavan, D.~Peleg, and E.~Upfal, ``Computing with noisy
  information,'' \emph{SIAM J. Comput.}, vol.~23, no.~5, pp. 1001--1018, 1994.

\bibitem{tree1}
S.~R. Kosaraju, T.~M. Przytycka, and R.~Borgstrom, ``On an optimal split tree
  problem,'' \emph{Lecture Notes in Computer Science: Algorithms and Data
  Structures}, vol. 1663, pp. 157--168, 1999.

\bibitem{kor1}
A.~P. Korostelev, ``On minimax rates of convergence in image models under
  sequential design,'' \emph{Statistics \& Probability Letters}, vol.~43, pp.
  369--375, 1999.

\bibitem{kor2}
A.~P. Korostelev and J.-C. Kim, ``Rates of convergence fo the sup-norm risk in
  image models under sequential designs,'' \emph{Statistics \& Probability
  Letters}, vol.~46, pp. 391--399, 2000.

\bibitem{active_vision}
M.~Swain and M.~Stricker, ``Promising directions in active vision,'' \emph{Int.
  J. Computer Vision}, vol.~11, no.~2, pp. 109--126, 1993.

\bibitem{geman}
D.~Geman and B.~Jedynak, ``An active testing model for tracking roads in
  satellite images,'' \emph{IEEE Trans. PAMI}, vol.~18, no.~1, pp. 1--14, 1996.

\bibitem{arkin:98}
E.~M. Arkin, H.~Meijer, J.~S.~B. Mitchell, D.~Rappaport, and S.~Skiena,
  ``Decision trees for geometric models,'' \emph{Intl. J. Computational
  Geometry and Applications}, vol.~8, no.~3, pp. 343--363, 1998.

\bibitem{freund:97}
Y.~Freund, H.~S. Seung, E.~Shamir, and N.~Tishby, ``Selective sampling using
  the query by committee algorithm,'' \emph{Machine Learning}, vol.~28, pp.
  133--168, 1997.

\bibitem{dasgupta:04}
S.~Dasgupta, ``Analysis of a greedy active learning strategy,'' in \emph{Neural
  Information Processing Systems}, 2004.

\bibitem{nina:07}
M.-F. Balcan, A.~Broder, and T.~Zhang, ``Margin based active learning,'' in
  \emph{Conf. on Learning Theory (COLT)}, 2007.

\bibitem{nowak_gbs}
R.~Nowak, ``Generalized binary search,'' in \emph{Proceedings of the Allerton
  Conference, Monticello, IL}, (www.ece.wisc.edu/$\sim$nowak/gbs.pdf) 2008.

\bibitem{tree4}
L.~Hyafil and R.~L. Rivest, ``Constructing optimal binary decision trees is
  {NP}-complete,'' \emph{Inf. Process. Lett.}, vol.~5, pp. 15--17, 1976.

\bibitem{tree3}
M.~R. Garey and R.~L. Graham, ``Performance bounds on the splitting algorithm
  for binary testing,'' \emph{Acta Inf.}, vol.~3, pp. 347--355, 1974.

\bibitem{cormen:09}
C.~Cormen, C.~Leiserson, R.~Rivest, and C.~Stein, \emph{Introduction to
  Algorithms, Third Edition}.\hskip 1em plus 0.5em minus 0.4em\relax MIT Press,
  2009.

\bibitem{pelc:02}
A.~Pelc, ``Searching games with erro r-- fifty years of coping with liars,''
  \emph{Theoretical Computer Science}, vol. 270, pp. 71--109, 2002.

\bibitem{angluin:88}
D.~Angluin, ``Queries and concept learning,'' \emph{Machine Learning}, vol.~2,
  pp. 319--342, 1988.

\bibitem{angluin:01}
------, ``Queries revisited,'' \emph{Springer Lecture Notes in Comp. Sci.:
  Algorithmic Learning Theory}, pp. 12--31, 2001.

\bibitem{kulkarni}
S.~R. Kulkarni, S.~K. Mitter, and J.~N. Tsitsiklis, ``Active learning using
  arbitrary binary valued queries,'' \emph{Machine Learning}, pp. 23--35, 1993.

\bibitem{hegedus:95}
T.~Heged{\"u}s, ``Generalized teaching dimensions and the query complexity of
  learning,'' in \emph{8th Annual Conference on Computational Learning Theory},
  1995, pp. 108--117.

\bibitem{hellerstein:96}
L.~Hellerstein, K.~Pillaipakkamnatt, V.~Raghavan, and D.~Wilkins, ``How many
  queries are needed to learn?'' \emph{J. ACM}, vol.~43, no.~5, pp. 840--862,
  1996.

\bibitem{traub88}
J.~F. Traub, G.~W. Wasilkowski, and H.~Wozniakowski, \emph{Information-based
  Complexity}.\hskip 1em plus 0.5em minus 0.4em\relax Academic Press, 1988.

\bibitem{z1}
K.~S. Zigangirov, ``Message transmission in a binary symmetric channel with
  noiseless feedback 9random transmission time),'' \emph{Probl. Peredachi
  Inform.}, vol.~4, no.~3, pp. 88--97, 1968.

\bibitem{kk:07}
R.~Karp and R.~Kleinberg, ``Noisy binary search and its applications,'' in
  \emph{Proceedings of the 18th ACM-SIAM Symposium on Discrete Algorithms (SODA
  2007)}, pp. 881--890.

\bibitem{rui:08}
R.~Castro and R.~Nowak, ``Minimax bounds for active learning,'' \emph{IEEE
  Trans. Info. Theory}, pp. 2339--2353, 2008.

\bibitem{rivest}
R.~L. Rivest, A.~R. Meyer, and D.~J. Kleitman, ``Coping with errors in binary
  search procedure,'' \emph{J. Comput. System Sci.}, pp. 396--404, 1980.

\bibitem{spencer}
J.~Spencer, ``Ulam's searching game with a fixed number of lies,'' in
  \emph{Theoretical Computer Science}, 1992, pp. 95:307--321.

\bibitem{dhagat}
A.~Dhagat, P.~Gacs, and P.~Winkler, ``On playing `twenty questions' with a
  liar,'' in \emph{Proc. ACM Symposium on Discrete Algorithms (SODA)}, 1992,
  pp. 16--22.

\bibitem{dhagat1}
J.~Aslam and A.~Dhagat, ``Searching in the presence of linearly bounded
  errors,'' in \emph{Proc. ACM Symposium on the Theory of Computing (STOC)},
  1991, pp. 486--493.

\bibitem{matti}
M.~K{\" a}{\" a}ri{\" a}inen, ``Active learning in the non-realizable case,''
  in \emph{Algorithmic Learning Theory}, 2006, pp. 63--77.

\bibitem{bremaud}
P.~Br\'{e}maud, \emph{Markov Chains, Gibbs Fields, Monte Carlo Simulations, and
  Queues}.\hskip 1em plus 0.5em minus 0.4em\relax Springer, 1998.

\bibitem{clay}
C.~Scott, personal communication.

\bibitem{candes:tit06a}
E.~J. Cand\`{e}s, J.~Romberg, and T.~Tao, ``Robust uncertainty principles:
  Exact signal reconstruction from highly incomplete frequency information,''
  \emph{{IEEE} Trans. Inform. Theory}, vol.~52, no.~2, pp. 489--509, Feb. 2006.

\bibitem{donoho:tit06}
D.~L. Donoho, ``Compressed sensing,'' \emph{{IEEE} Trans. Inform. Theory},
  vol.~52, no.~4, pp. 1289--1306, 2006.

\bibitem{dasgupta:05}
S.~Dasgupta, ``Coarse sample complexity bounds for active learning,'' in
  \emph{Neural Information Processing Systems}, 2005.

\bibitem{buck:43}
R.~C. Buck, ``Partition of space,'' \emph{The American Mathematical Monthly},
  vol.~50, no.~9, pp. 541--544, 1943.

\bibitem{vapnik95}
V.~Vapnik, \emph{The Nature of Statistical Learning Theory}.\hskip 1em plus
  0.5em minus 0.4em\relax NY: Springer, 1995.

\end{thebibliography}
\end{document}